\documentclass[journal]{IEEEtran}

\usepackage{url}

\usepackage{array,multirow}
\usepackage{amsthm}
\usepackage{graphicx, xcolor} 
\usepackage{enumerate}

\usepackage[disable]{todonotes}

\usepackage{tikz}
\usepackage{amsmath} 
\usepackage{amssymb, dsfont}  

\newtheorem{example}{Example}

\newtheorem{defn}{Definition}

\newtheorem{Problem}{Problem}
\newtheorem{prop}{Theorem}

\newtheorem{corr}{Corollary}
\newtheorem{rem}{Remark}

\def\beq{\begin{equation}}
\def\eeq{\end{equation}}

\newcommand{\br}{\mathbb{R}}

\newcommand{\bn}{\mathbb{N}}

\newcommand{\z}[3]{z_{#1}^{#2}(#3)}
\newcommand{\zr}[3]{r_{#1}^{#2}(#3)}
\newcommand{\y}[2]{y^{#1}(#2)}
\newcommand{\s}[2]{\pi_{#1}(#2)}
\newcommand{\uu}[2]{u_{#1}(#2)}
\newcommand{\w}[2]{w_{#1}(#2)}

\newcommand{\bigt}[1]{T_{#1}}
\newcommand{\traj}[1]{\pi_{#1}}
\newcommand{\ktn}[2]{{k_{#1}(#2)}}
\newcommand{\tildez}[3]{\tilde{z}^{#2}_{#1}(#3)}
 
\newcommand{\tildey}[2]{\tilde{y}^{#1}(#2)}
\newcommand{\bary}[2]{\bar{y}^{#1}(#2)}

\newcommand{\ra}{\rightarrow}
\newcommand{\Next}{\bigcirc}
\newcommand{\release}{\hspace{1mm}\mathcal{R}\hspace{1mm}}
\newcommand{\until}{\hspace{1mm}\mathcal{U}\hspace{1mm}}
\newcommand{\true}{\relax\ifmmode \mathit{True} \else \em True \/\fi}
\newcommand{\false}{\relax\ifmmode \mathit{False} \else \em False \/\fi}
\newcommand{\eventually}{\lozenge}
\newcommand{\always}{\Box}

\begin{document}

\title{Multirobot Coordination with Counting Temporal Logics}

\author{\IEEEauthorblockN{Yunus Emre Sahin\IEEEauthorrefmark{1},
Petter Nilsson\IEEEauthorrefmark{2},
Necmiye Ozay\IEEEauthorrefmark{1}
}\\
\IEEEauthorblockA{\IEEEauthorrefmark{1}
	Department of Electrical Engineering and Computer Science,\\
University of Michigan, Ann Arbor, MI 48109 USA,\\ 
\IEEEauthorrefmark{2}
Department of Mechanical and Civil Engineering,\\
California Institute of Technology, Pasadena CA 91125 USA\\}
\thanks{This work is supported in part by NSF grants CNS-1239037, CNS-1446298 and ECCS-1553873, and DARPA grant N66001-14-1-4045. Emails: {\tt\small \{ysahin,necmiye\}@umich.edu}, {\tt\small pettni@caltech.edu}.}}


\IEEEtitleabstractindextext{%
\begin{abstract}
  In many multirobot applications, planning trajectories in a way to guarantee that the collective behavior of the robots satisfies a certain high-level specification is crucial. Motivated by this problem, we introduce \emph{counting temporal logics}---formal languages that enable concise expression of multirobot task specifications over possibly infinite horizons. We first introduce a general logic called \emph{counting linear temporal logic plus} (cLTL+), and propose an optimization-based method that generates individual trajectories such that satisfaction of a given cLTL+ formula is guaranteed when these trajectories are synchronously executed. We then introduce a fragment of cLTL+, called counting linear temporal logic (cLTL), and show that a solution to planning problem with cLTL constraints can be obtained more efficiently if all robots have identical dynamics. In the second part of the paper, we relax the synchrony assumption and discuss how to generate trajectories that can be asynchronously executed, while preserving the satisfaction of the desired cLTL+ specification. In particular, we show that when the asynchrony between robots is bounded, the method presented in this paper can be modified to generate robust trajectories. We demonstrate these ideas with an experiment and provide numerical results that showcase the scalability of the method.
\end{abstract}

}

\maketitle

\IEEEdisplaynontitleabstractindextext

\IEEEpeerreviewmaketitle

\section{Introduction}

\todo[inline]{It would be a good idea to write about how solution times are sensitive to encoding methods in the introduction, to motivate why so many variants are introduced. Would also be good to add some intuitive explanations around the encodings, for a part there are just new encoding equations without too much explanation.}

\IEEEPARstart{M}{ultirobot} systems can serve modern societies in a variety of ways, ranging from pure entertainment \cite{arikan2001efficient,jolly2009bezier} to critical search and rescue missions \cite{nagatani2013emergency, macwan2015multirobot}, from construction automation \cite{petersen2011termes} to micromanipulation \cite{kantaros2018control}. {The number of robots required to achieve a common goal increases each day to improve the effectiveness and efficiency in such applications}. Therefore, there is a need for scalable tools to coordinate the collective behavior of large numbers of robots. In this paper, we introduce \emph{counting temporal logics} for specifying desired collective behavior of multirobot systems in a concise manner, and provide an optimization-based algorithm to synthesize trajectories that ensure the satisfaction of specifications given in this formalism. We show that counting temporal logics can capture meaningful and interesting multirobot tasks, and that the solution method proposed in this paper scales better with the number of robots than the existing methods. In fact, we show that our method scales to hundreds of robots under certain conditions. Moreover, we do not require robots to be synchronized perfectly or communicate during runtime. 

Traditional algorithms for multirobot coordination tend to focus on relatively simple tasks such as reaching a goal state while avoiding unsafe regions and collisions \cite{saha2016implan,yu2016optimal, wang2017safety}, or reaching a consensus \cite{hu2012robust, mesbahi2010graph}. Temporal logics, such as Linear Temporal Logic (LTL), provide a powerful framework for defining more complex specifications, for example: \emph{Always avoid collision with obstacles, do not cross into region A before visiting region B, and eventually visit regions A and C repeatedly}. Given requirements in a formal language, existing methods such as \cite{bhatia2010sampling,kress2009temporal,wongpiromsarn2010formal} can generate correct-by-construction trajectories for single-agent systems. The use of LTL specifications has also been considered for multirobot systems \cite{chen2012formal,guo2015multi,kloetzer2010automatic, moarref2017decentralized,raman2014reactive,ulusoy2012robust}. However, generalizations to multirobot systems suffer from the curse of dimensionality and cannot handle large numbers of robots. Furthermore, LTL does not provide a natural way to define group tasks, hence using LTL in multirobot settings results in long formulas, which are not desired as the complexity of the algorithms depend on the length of the formula.

Existing methods that use temporal logic to define multirobot specifications, such as \cite{guo2015multi,ulusoy2012robust}, require that each robot be assigned an independent task, a tedious and error-prone process when the number of robots is large. In many applications, completion of a task depends not on identities of robots, but on the number of robots satisfying a property. Take for example an emergency response scenario where hundreds of autonomous vehicles are deployed to locate and help the victims. In such a scenario, it is reasonable to assume that most of the vehicles would have identical capabilities and that the identity of the vehicle is not important to the rescuers, as long as the given tasks are accomplished. On the other hand, tasks might depend on the number of agents satisfying a property. For instance, one might require sufficiently many robots to surveil a particular area to look for victims. Or, one might need to limit the number of rescuers in certain regions to avoid unsafe areas or congestion. We call this type of specification \emph{temporal counting constraints} and propose a novel logic called \emph{counting linear temporal logic plus (cLTL+)} to specify them. This logic is two-layered similar to \cite{xu2016census}. The inner logic defines tasks that can be satisfied by a single robot, for instance \emph{surveiling an area} in the previous emergency response scenario. The outer logic requires \emph{sufficiently many} (or \emph{not too many}) robots to satisfy tasks given as inner logic formulas. For example, one might express a task that ``\emph{at least $2$ and not more than $5$ robots} to surveil an area'' using cLTL+.

After introducing the logic, we propose an optimization-based method to generate individual trajectories that collectively satisfy specifications given in cLTL+. The method proposed in this paper uses an integer linear programming (ILP) formulation of temporal specifications with the assumption that robots are perfectly synchronized. We later relax this assumption and show how to generate solutions robust to bounded synchronization errors.

We also discuss several variants of the cLTL+ syntax. Firstly, we introduce a fragment of cLTL+, namely \emph{counting linear temporal logic (cLTL)}. We show that an alternative solution method could scale to systems with hundreds of robots when specifications are given in cLTL and robots have identical dynamics. The logic cLTL and associated synthesis algorithms can be seen as an extension of a special class of counting problems that deal with invariant specifications, first proposed in \cite{nilsson2016control,nilsson2018control}. Secondly, we present an extension to the syntax of cLTL+ to define tasks that could be carried out only by a certain group of robots. For example, one might require a surveillance task to be conducted by robots that are equipped with suitable cameras. This extension allows us to assign tasks to specific group of robots. Finally, we show that continuous state dynamics can be handled directly within our framework.

As another contribution of this paper, we discuss how to relax the synchronous execution assumption and generate trajectories that can be executed asynchronously. Robustness against noise and parameter uncertainty has been extensively studied for single robot systems \cite{zhou1998essentials}, and also extended to consensus problems \cite{6160957}. However, additional factors need to be addressed when dealing with multirobot systems. Unlike single robot systems, multirobot systems might tolerate the failure of individual agents without sacrificing task fulfillment. Such a notion of robustness against failing robots is examined in \cite{cortes2006robust,Kumar2000TFM,sahin2017provably}. Another consideration in multirobot coordination problems is the robustness against synchronization errors. Perfect synchronization of robots might not be practical in real-life applications. The authors of \cite{ulusoy2012robust} characterized a class of LTL formulas that are robust to asynchrony and provided bounds on the deviation from optimality in the presence of asynchrony. However, for general LTL specifications, correctness cannot be guaranteed using this approach. A method that is based on prioritizing robots and planning individual trajectories sequentially was recently proposed in \cite{Desai:2017:DFS}. Trajectories generated with this approach, however, depend highly on how the robots are prioritized---feasible solutions can be missed if priorities are not correctly assigned. In this paper we propose a new definition of robust satisfaction of temporal logic formulas, similar in spirit to \cite{donze2010robust}. We then provide small modifications to our method to generate trajectories that satisfy this notion of robustness, and show that the method is sound and partially complete.

Preliminary versions of this paper appeared in \cite{sahin2017provably} and \cite{sahinsynchronous}. This paper provides a more comprehensive treatment of counting temporal logics and corresponding synthesis problems, including partially complete robust encodings, full proofs and several extensions. Moreover, experimental results implementing the synthesized trajectories in Robotarium \cite{pickem2017robotarium} are provided. The rest of the paper is organized as follows. Background information is provided in Section II. Section III introduces the syntax and semantics for cLTL+ and cLTL. Section IV formally defines the synchronous coordination problem and proposes a solution. An alternative solution, which can solve a special set of problems more efficiently, is also provided in the same section. Section V introduces a time-robustness concept and presents necessary modifications to the method in order to generate robust solutions. Section VI presents two extensions. We demonstrate the efficacy of the methods presented in this paper via numerical and experimental results in Section VII before concluding the paper in Section VIII.

\section{System and behavior descriptions}

This section introduces the notation used in the rest of the paper and provides system and behavior definitions required to formally state the problem we seek to solve.

The set of nonnegative integers is denoted by $\bn$ and the set of positive integers up to $N$ is denoted by $[N] = \{1,2,\dots,N\}$. We use $\mathbf{1}$ to denote the 
vector of all $1$'s. 
We define a set membership indicator function such that given a set A, $\mathds{1}_A(a) = 1$ if $a \in A$ and $\mathds{1}_A(a) = 0$ otherwise. The cardinality of a set $A$ is denoted by $|A|$. We next define transition systems that are used to model the robot dynamics.

\begin{defn} A \textbf{transition system} is a tuple $T = (S, \ra, AP, L)$ where $S$ is a finite set of states, $\ra \subseteq S \times S$ is a transition relation, $AP$ is a finite set of atomic propositions, and $L: S\ra 2^{AP}$ is a labeling function.
\end{defn}

We say that \emph{$s$ satisfies $a$} or \emph{$a$ holds at $s$} if $a\in L(s)$ for $s\in S$ and $a\in AP$. A transition system is said to be \emph{action deterministic} if all transitions are controllable. In this work, we assume that robot dynamics are modeled by action deterministic transition systems. This implies that, if the transition relation includes $(s,s')$, then there exists a controller that can steer a robot from state $s \in S$ to state $s' \in S$. Action deterministic transition systems could capture the behavior of many complex systems and could be obtained using abstraction methods \cite{pola2008approximately,wongpiromsarn2010formal} or motion primitives \cite{gray2012predictive,mellinger2011minimum, paranjape2015motion}. Such abstract graph-based representations are commonly used for describing the behavior of robotic teams \cite{banfi2018multirobot,yu2016optimal}.

\begin{defn}\label{def:trajectory} Given a transition system $T = (S, \ra, AP, L)$, an infinite sequence $\pi: \s{}{0}\s{}{1}\s{}{2}\ldots \in S^\omega$ of states such that $(\s{}{k}, \s{}{k+1})\in \ra$ is called a \textbf{trajectory}. For a given trajectory $\pi$, the corresponding \textbf{trace} is defined as $\sigma(\pi) =  L(\s{}{0})L(\s{}{1})L(\s{}{2}) \ldots \in (2^{AP})^\omega$. 
\end{defn}
The transition system and the trajectories associated with robot {$\mathcal{R}_n$} are denoted by $\bigt{n} = (S_n, \ra_n, AP, L_n)$ and $\traj{n}$, respectively. As indicated by this notation, we allow the dynamics of robots to differ but require that they share the same atomic propositions. Note that this requirement could be achieved without loss of generality, as one can define a global atomic proposition set simply by taking the union of all atomic propositions. For a collection $\{\bigt{n}\}_{n\in[N]}$ of transition systems (or a collection $\{\traj{n}\}_{n\in[N]}$ trajectories), we drop ${n\in[N]}$ and write $\{\bigt{n}\}$ (or $\{\traj{n}\}$) when the range of $n$ is clear from the context.

The collective behavior of a multirobot system depends not only on the individual trajectories but also on how they are interleaved. If robots are not synchronized, there are infinitely many ways a collection of trajectories could be executed. Depending on how the asynchrony plays out, a given property might or might not be satisfied by a given collection of trajectories. Since it is difficult to synchronize a large number of robots perfectly in practice, we allow robots to move asynchronously. To reason about asynchronous executions, we define \emph{local counters}:
\begin{defn}\label{local_counter}
   A mapping $k:\bn\to\bn$ is called a \textbf{local counter} if it satisfies the following:
  \begin{equation}\label{eq:collective}
  \ktn{}{0} =0,\; \ktn{}{t} \leq \ktn{}{t+1} \leq \ktn{}{t}+1, \; \lim_{t\rightarrow\infty} \ktn{}{t} = \infty.
  \end{equation}
  The set of all local counters is denoted by $\mathcal{K}$.
\end{defn}
A local counter is used to keep track of how far a robot has moved along its trajectory. If $\pi_n$ denotes the trajectory and $k_n$ denotes the local counter of robot $\mathcal{R}_n$, the position of $\mathcal{R}_n$ at time $t$ is given by $\traj{n}(\ktn{n}{t})$. Equation \eqref{eq:collective} guarantees that initial conditions are respected, the order of states in a trajectory is preserved, and that robots eventually make progress. 

Given a collection of trajectories, a particular execution is uniquely identified by local counters:
\begin{defn}\label{execution}
 An $N$-dimensional \textbf{collective execution} $K: \bn \to \bn^N$ is a mapping from global time to local counters, i.e., {$K\doteq [k_1 \dots k_N]$} where $k_n \in \mathcal{K}$ for all $n\in [N]$. The set of all $N$-dimensional collective executions is denoted by $\mathcal{K}_N$.
\end{defn}
For a collection $\Pi = \{\traj{1},\dots,\traj{N}\}$ of trajectories and a collective execution $K$, we use $(\Pi,K)$ to denote the unique execution of the trajectories corresponding to $K$. To illustrate the concept of collective execution, we present the following example:

\begin{figure}[t]
  \centering 
  \includegraphics[width=.95\linewidth]{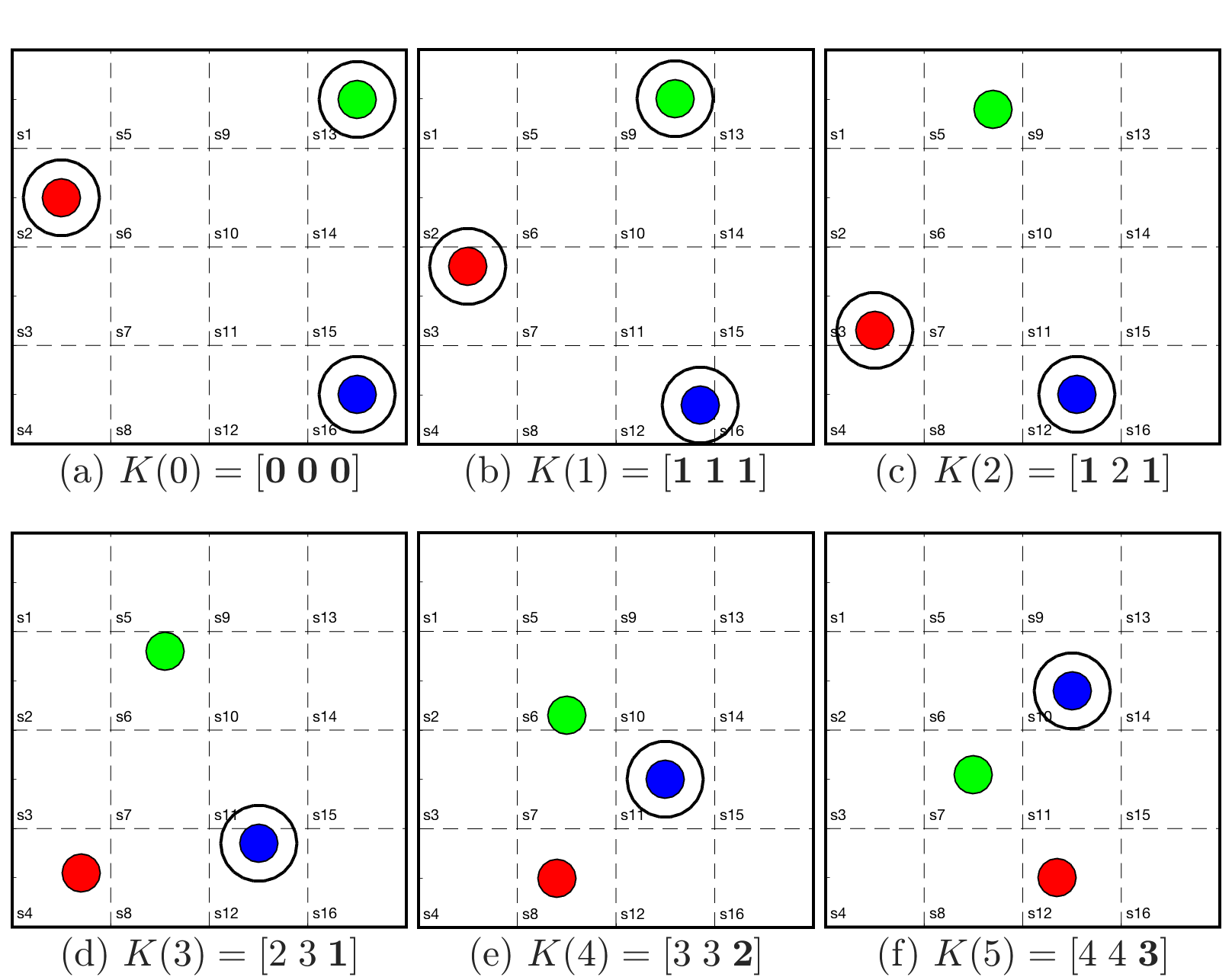}
  \caption{Frames (a) to (e) correspond to snapshots of a possible asynchronous execution taken at times $t=0$ to $t=5$. Robots are enumerated in the order of red, green, blue and local times of robots at each time step are shown below the corresponding frame. Anchoring robots are highlighted with a black circle and the anchor time is shown in bold.} 
  \label{fig:running_example} 
\end{figure}

\begin{example}
\label{ex:running_example}
  Let the following three trajectories
    \begin{align*}
    \begin{matrix}
    \traj{1} &= \\
    \traj{2} &= \\  
    \traj{3} &= 
    \end{matrix} \quad
    \begin{matrix}
    s_2& s_3 & s_4 & s_8 & s_{12} &\dots\\
    s_{13}& s_9& s_{5}&s_6 &s_7&\dots\\ 
    s_{16}& s_{12}& s_{11}& s_{10} & s_9&\dots.
    \end{matrix}
    \end{align*}
  denote the trajectories of a red, green, and a blue robot, respectively. An arbitrary collective execution is illustrated in Figure \ref{fig:running_example}. Local counters are initially set as $K(0) = [0\; 0\; 0]$ at time $t=0$; that is, each robot $\mathcal{R}_n$ is initially positioned at $\s{n}{0}$. Every robot completes a transition by time $t=1$, so local counters are updated as $K(1) = [1\; 1\; 1]$. The red and the blue robots move slower than expected and fail to complete two transitions by time $t=2$. The green robot, on the other hand, successfully completes two transitions by time $t=2$. Thus, local counters are updated as $K(2) = [1\; 2\; 1]$. Similarly, the values of the local counters up to $t=5$ can be seen from Figure~\ref{fig:running_example}.
\end{example}
 
As stated before, when robots are allowed to move asynchronously, there are infinitely many collective executions given a collection of trajectories. Without a bound on asynchrony, it might be impossible to achieve meaningful tasks. For this reason, we introduce the following definition.

\begin{defn}  
A collective execution $K = [k_1\dots k_N]$ is called \textbf{$\boldsymbol{\tau}$-bounded} if 
  $$
  \max_{t\in\bn,n,m\in[N]}(|\ktn{n}{t}-\ktn{m}{t}|)\leq \tau.
  $$ 
  The set of all $\tau$-bounded $N$-dimensional collective executions is denoted by $\mathcal{K}_{N}(\tau)$.
\end{defn}

A collective execution $K \in \mathcal K_N(0)$ is called a \emph{synchronous} execution. In a synchronous execution, all robots start and complete their transitions simultaneously. The synchronous execution $K^* = [k^*_1 \dots k^*_N]$ where $k^*_n(t)=t$ for all $n$ and $t$ is called \emph{globally synchronous}.

\section{Counting logics: syntax and semantics}

This section provides the syntax and semantics of \emph{counting linear temporal logic plus} (cLTL+), as well as the smaller fragment \emph{counting linear temporal logic} (cLTL) which allows for more efficient solutions under certain conditions. 

\subsection{cLTL+}

The logic cLTL+ is a two-layer logic similar to censusSTL \cite{xu2016census}. The \emph{inner logic} is identical to LTL and is used to describe tasks that can be satisfied by a single robot. For example, tasks such as \emph{``avoid collisions with obstacles at all times"} or \emph{``eventually visit region $A$"} can be described by the inner logic. The outer layer then specifies the evolution of the number of robots required to satisfy an inner logic formula. Using the earlier examples, we can specify tasks such as \emph{``All robots} must avoid collisions with obstacles" or \emph{``At least five robots} should eventually visit region $A$" using cLTL+. 

An inner logic formula over a set $AP$ of atomic propositions is defined recursively as follows:
\beq\label{eq:inner}
 \phi ::= True \mid ap  \mid \neg \phi \mid \phi_1\land \phi_2 \mid \Next \phi \mid \phi_1 \until \phi_2, 
\eeq
where $ap \in AP$ is an atomic proposition and $\phi, \phi_1$ and $\phi_2$ are inner logic formulas. 
The symbols $\neg, \land,\Next$ and $\until$ correspond to the logical operators \emph{negation} and \emph{conjunction}, and the temporal operators \emph{next} and \emph{until}, respectively. Other commonly used operators can be derived from these operators, such as \emph{disjunction} $\left(\phi_1 \lor \phi_2 \doteq  \neg (\neg \phi_1 \land \neg \phi_2)\right)$, \emph{release} $\left(\phi_1 \release\phi_2 \doteq \neg\left(\neg\phi_1\until\neg\phi_2\right)\right)$, \emph{eventually} $\left(\eventually\phi\doteq True \until \phi\right)$, \emph{always} $\left(\always\phi\doteq \neg (\eventually\neg \phi)\right)$, etc. We use $\Phi$ to denote the set of all inner logic formulas defined according to \eqref{eq:inner}. Although the inner logic is identical to LTL, we present the semantics here for the sake of completeness.

Let $\sigma \in (2^{AP})^\omega$ be a trace and let $\phi$ be an inner logic formula. Satisfaction of $\phi$ by $\sigma$ at step $t$ is denoted by $\sigma,t \models \phi$ and is defined as follows:
\begin{itemize}
  \item $\sigma,t\models \true$,
  \item for any atomic proposition $a\in AP$, $\sigma,t\models a$ if and only if $a\in \sigma(t)$,
  \item $\sigma,t \models \varphi_1 \land \varphi_2$ if and only if $\sigma,t \models \varphi_1$ and $\sigma,t \models \varphi_2$,
  \item $\sigma,t \models \neg\varphi$ if and only if $\sigma,t \not\models \varphi$,
  \item $\sigma,t \models  \bigcirc\varphi$ if and only if $\sigma,t+1 \models \varphi$, and
  \item $\sigma,t \models \varphi_1 \until \varphi_2$ if and only if there exists $l\geq 0$ such that $\sigma,t+l \models \varphi_2$ and $\sigma,t+l' \models \varphi_1$ for all $0\leq l' <l$.
\end{itemize}

If $\sigma,0 \models \varphi$, then we say that $\sigma$ \emph{satisfies} $\varphi$ and write $\sigma \models \varphi$ for short. We say that a trajectory $\pi$ satisfies $\varphi$ if $\sigma(\pi)\models \varphi$, and write $\pi \models \varphi$.

After defining the inner logic, we now present the syntax for cLTL+ which is based on a new proposition type: a \emph{temporal counting proposition} ($tcp$) is an inner logic formula paired with a nonnegative integer, i.e., $tcp = [\phi, m] \in \Phi \times \bn$. 
The inner logic formula $\phi$ defines a task and $m$ specifies the number of robots needed to satisfy it.
For example, $tcp = [\eventually a, 5]$ is a temporal counting proposition that evaluates to $True$ if the task ``$\eventually a$'' is satisfied by at least five robots. 

The following grammar can now be used to recursively define cLTL+ formulas:
\beq\label{eq:outer}
\mu ::= True \mid tcp  \mid \neg \mu \mid \mu_1\land \mu_2 \mid \Next \mu \mid \mu_1 \until \mu_2,
\eeq
where $tcp \in \Phi \times \bn$ is a temporal counting proposition and $ \mu ,\mu_1$ and $ \mu_2$ are cLTL+ formulas. Identical to inner logic, other commonly used operators can be derived from \eqref{eq:outer}.

Let $\Pi = \{\traj{1},\dots,\traj{N}\}$ be a collection of trajectories and $K = [k_1 \dots k_N]$ be a collective execution. Semantics of the outer logic is similar to the semantics of the inner logic, but they are defined for executions of collections of trajectories. Satisfaction of a cLTL+ formula $\mu$ by the pair $(\Pi,K)$ at time $t$, denoted as $(\Pi,K),t \models \mu$, is defined as follows: 

\begin{itemize}
  \item $(\Pi,K),t\models \true$,
  \item for any temporal counting proposition $tcp = [\phi, m]\in \Phi \times \bn$, we say $(\Pi,K),t\models tcp$ if and only if $|\{n \mid \sigma(\pi_n),k_n(t) \models \phi\}| \geq m$,
  \item $(\Pi,K),t \models \mu_1 \land \mu_2$ if and only if $(\Pi,K),t \models \mu_1$ and $(\Pi,K),t \models \mu_2$,
  \item $(\Pi,K),t \models \neg\mu$ if and only if $(\Pi,K),t \not\models \mu$,
  \item $(\Pi,K),t \models  \bigcirc\mu$ if and only if $(\Pi,K),t+1 \models \mu$, and
  \item $(\Pi,K),t \models \mu_1 \until \mu_2$ if and only if there exists $l\geq 0$ such that $(\Pi,K),t+l \models \mu_2$ and $(\Pi,K),t+l' \models \mu_1$ for all $0\leq l' <l$.
\end{itemize}
If $(\Pi,K),0 \models \mu$, then we say that the pair $(\Pi,K)$ \emph{satisfies} $\mu$ and write $(\Pi,K) \models \mu$ for short.

\subsection{cLTL}

Having defined the cLTL+, we now introduce \emph{counting linear temporal logic} (cLTL), which corresponds to the fragment of cLTL+ where the inner logic is constrained to the grammar $\phi ::= a$. 
Temporal counting propositions in cLTL have the special form $tcp_{cLTL} = [a,m]$ where the inner logic is restricted to atomic propositions instead of an LTL formula, i.e., $a \in AP$. As a result of this restriction, cLTL enforces robots to ``synchronize''.
The following example depicts the differences between cLTL and cLTL+ formulas: 

\begin{example}
  Consider the following cLTL+ formulas: $\mu_1 \doteq \always\eventually [a, m]$, $\mu_2 \doteq [\always\eventually a, m]$, and $\mu_3 \doteq \always[\eventually a, m]$ for $a \in AP$.

  Here the inner formula of $\mu_1$, ``$a$'', is an atomic proposition. Hence, $\mu_1$ is also a cLTL formula where the task ``$a$'' can be satisfied by any robot, simply by visiting a state where $a$ holds. The temporal counting proposition ``$[a,m]$'' is satisfied at time $t$ if at least $m$ robots to satisfy $a$ at time $t$. Moreover, the temporal operators ``$\always\eventually$'' in the outer layer necessitate that the temporal counting proposition is satisfied infinitely many times. Thus, there should be an infinite number of instances where $a$ is \emph{simultaneously} satisfied by more than $m$ robots in order for $\mu_1$ to be satisfied.
    
  On the other hand, neither $\mu_2$ nor $\mu_3$ can be specified in cLTL. In both formulas, the inner formula contains temporal operators which are not allowed in the cLTL syntax. The difference between $\mu_1$ and $\mu_2$ is that the latter relaxes the simultaneity requirement. The inner formula $\always\eventually a$ can be satisfied by any robot if the robot satisfies $a$ infinitely many times. The integer $m$ is the smallest number of robots that needs to satisfy the inner formula. Hence, the cLTL+ formula $\mu_2$ requires at least $m$ robots to satisfy $a$ infinitely many times, but as opposed to $\mu_1$ they need not do so simultaneously. For any given time the number of robots that satisfy $a$ might never exceed $m$, or even $1$. Note that any collective trajectory that satisfies $\mu_1$ also satisfies $\mu_2$, but the converse is not true. 
    
  The difference between $\mu_2$ and $\mu_3$ is more subtle. Any collective trajectory that satisfies $\mu_2$ would also satisfy $\mu_3$. The converse is also true if the number of robots is finite. However, in the hypothetical scenario where there are infinitely many robots, $\mu_3$ can be satisfied even if no robot satisfies $a$ more than once. 
  $\blacksquare$ 
\end{example}

\section{Synchronous coordination problem and its solution}\label{chp:sol}

This section provides the formal definition of the synchronous multirobot coordination problem and provides an optimization-based solution for cLTL+ specifications. Subsequently, an alternative solution is proposed for the special case where the specifications are given in cLTL and the robots have identical dynamics. The alternative solution is shown to scale much better with the number of robots. In fact, the number of robots has almost no effect on the solution time and problems with hundreds of robots can be solved with the alternative method as demonstrated in Section \ref{chp:results}.

\begin{Problem}\label{prob1}
  Given $N$ robots with dynamics $\{\bigt{n} = (S_n,\ra_n, AP, L_n)\}$, initial conditions $\{\s{n}{0}\}$, and a cLTL+ formula $\mu$ over $AP$, synthesize a collection $\Pi = \{\traj{1},\dots, \traj{N}\}$ such that the globally synchronous execution of $\Pi$ satisfies $\mu$, i.e., $(\Pi,K^*) \models\mu$.
\end{Problem}

In order to solve Problem \ref{prob1}, we generate individual trajectories in a centralized fashion. Robots then follow these trajectories in a distributed fashion, using local controllers without runtime communication. To generate trajectories we encode the robot dynamics and the cLTL+ constraints using integer linear constraints and pose the synthesis problem as an integer linear program (ILP). This approach is inspired by the bounded model-checking literature \cite{biere2linear}. In particular, we focus the search on individual trajectories on prefix-suffix form. That is, for a given integer $h$, we aim to construct individual trajectories of the form $\traj{n}= \s{n}{0}\s{n}{1}\ldots \s{n}{h}\ldots$ and find an integer $l\in\{0,\ldots, h-1\}$ such that  for all $k \geq h$, $\s{n}{k} = \s{n}{k+l-h}$. In the following, we present ILP encodings of dynamic and temporal constraints.

\subsection{Globally synchronous robot dynamics}

Given the transition system $\bigt{n} = (S_n, \ra_n, AP, L_n)$ that represents the dynamics of robot $\mathcal{R}_n$, consider the adjacency matrix $A_n$ corresponding to the transition relation $\ra_n$. We use a Boolean vector $\w{n}{t} \in \{0,1\}^{|S_n|}$ with a single nonzero component to denote the state of robot $\mathcal{R}_n$ at time $t$. For example, assume $S_n = \{v^1, v^2, v^3\}$ and that robot $\mathcal{R}_n$ is at $v^2$ at time $t$. Then,  $\w{n}{t} = \begin{bmatrix}
 0 & 1& 0
 \end{bmatrix}^T$. With a slight abuse of notation, we equivalently write $\w{n}{t} = v^2$.

Given adjacency matrices $\{ A_n \}$ corresponding to $\{ \bigt{n} \}$ and a set of inital conditions $\{\s{n}{0}\}$, the dynamics of robot $\mathcal{R}_n$ are captured as follows:
\begin{equation}\label{eq:dynamics}
\begin{split}
\w{n}{t+1} &\leq A_n\w{n}{t}, \\
\w{n}{0} = \s{n}{0},&\quad \mathbf{1}^T \w{n}{t} = 1,
\end{split}
\end{equation}
for all  $n \in [N]$ and for all $t\in\{0,\ldots, h-1\}$. The trajectory $\traj{n}$ corresponding to the sequence $\mathbf{w}_n = \w{n}{0} \w{n}{1} \dots$ can then be extracted by locating the nonzero component in each $\w{n}{t}$.

\subsection{Loop constraints}

To ensure that the generated trajectories are in prefix-suffix form, we introduce $h$ binary variables $\mathbf{z_{loop}} = \{\z{loop}{}{0}, \dots \z{loop}{}{h-1}\}$ and the following constraints:
\begin{subequations}\label{eq:loop}
\begin{align}
\w{n}{h} &\leq  \w{n}{t} + \mathbf{1}(1-\z{loop}{}{t}), \\
\w{n}{h} &\geq  \w{n}{t} - \mathbf{1}(1-\z{loop}{}{t}),\\
\sum_{t=0}^{h-1} \z{loop}{}{t} &=  1
\end{align}
\end{subequations} 
 for all  $n \in [N]$ and for all $t \in \{0,\dots, h-1\}$. These constraints guarantee that there exists a unique $t$ such that $\z{loop}{}{t}=1$ and $\w{n}{h} =  \w{n}{t}$. For all other time instances, the first two inequalities are trivially satisfied.
 
\subsection{Inner logic constraints} 

We next recursively describe how counting temporal logic constraints can be translated into integer constraints. Let $\phi \in \Phi$ be an inner logic formula given according to \eqref{eq:inner} and $h$ be the horizon length. For each robot $n$, we introduce $h$ binary decision variables $\z{n}{\phi}{t} \in \{0,1\}$ for $t \in \{0,1,\dots,h-1\}$ and ILP constraints such that $\z{n}{\phi}{t} = 1$ if and only if $\pi_n,t \models \phi$. Hence, satisfaction of an inner formula $\phi$ by the robot $\mathcal{R}_n$ is equivalent to $\z{n}{\phi}{0} = 1$. We use the following encodings to recursively create the corresponding ILP constraints:

\emph{ap (atomic proposition):} Let $\phi = a \in AP$ be an atomic proposition and let the states of $\bigt{n}$ be given by the set $S_n = \{v_n^1,v_n^2,\dots,v_n^{|S_n|}\}$. We define the vector $\mathbf{v}_n^{\phi} \in \{0,1\}^{|S_n|}$ such that the $i^{th}$ entry of $\mathbf{v}_n^{\phi}$ is $1$ if and only if $a \in L(v_n^i)$. That is, $\mathbf{v}_n^\phi$ encodes the labeling function $L_n$. Then we introduce the following constraints for all $n \in [N]$:
\begin{equation}
\label{eq:ap}
\begin{split}
(\mathbf{v}_n^{\phi})^T \w{n}{t} &\geq \z{n}{\phi}{t},\\
(\mathbf{v}_n^{\phi})^T \w{n}{t} & < \z{n}{\phi}{t} +1.
\end{split}
\end{equation}
 
\emph{$\neg$ (negation):} Let $\phi = \neg \varphi$. Then for all $n \in [N]$,
\begin{equation}\label{eq:neg}
\z{n}{\phi}{t} = 1 - \z{n}{\varphi}{t},\qquad t=0,\dots,h-1.
\end{equation}

\emph{$\land$ (conjunction):} Let $\phi = \bigwedge_{i=1}^I \varphi_i$. Then for all $t=0,\dots,h-1$ and for all $n \in [N]$, 
\begin{equation}
\label{eq:conjunction}
\begin{split}
\z{n}{\phi}{t} & \leq \z{n}{\varphi_i}{t}, \qquad \text{for }  i=1,\dots,I \quad \text{and,}\\
\z{n}{\phi}{t} & \geq 1-I+\sum_{i=1}^I \z{n}{\varphi_i}{t}.
\end{split}
\end{equation}

\emph{$\lor$ (disjunction):} Let $\phi = \bigvee_{i=1}^I \varphi_i$. Then for all $t=0,\dots,h-1$ and for all $n \in [N]$,
\begin{equation}\label{eq:lor}
\begin{split}
\z{n}{\phi}{t} & \geq \z{n}{\varphi_i}{t}, \qquad \text{for }  i=1,\dots,I \quad \text{and,}\\
\z{n}{\phi}{t} & \leq \sum_{i=1}^I \z{n}{\varphi_i}{t}.
\end{split}
\end{equation}

With a slight abuse of notation, we also use Boolean operators on these optimization variables. For example, for $\phi = \bigvee_{i=1}^I \varphi_i$, we write $\z{n}{\phi}{t} = \bigvee_{i=1}^I \z{n}{\varphi_i}{t}$ instead of stating the inequalities in \eqref{eq:lor}. Encoding of the temporal operators is then as follows:

\emph{$ \bigcirc$ (next):} Let $\phi =  \bigcirc \varphi$, then for all $n \in [N]$
\begin{equation}
\begin{split}
\z{n}{\phi}{t} &= \z{n}{\varphi}{t+1}
, \qquad t=0,\dots,h-2 \text{ and, }\\
\z{n}{\phi}{h-1} &= \bigvee_{t=0}^{h-1} (\z{n}{\varphi}{t} \land \z{loop}{}{t}).
\end{split}
\end{equation}

\emph{$\until$(until):} if $\phi = \varphi_1 \until \varphi_2$, then for all $n \in [N]$
\begin{equation}\label{eq:until}
\begin{split}
\z{n}{\phi}{t} &= \z{n}{\varphi_2}{t} \lor \left( \z{n}{\varphi_1}{t} \land \z{n}{\phi}{t+1}  \right), \quad t\leq h-2,\\
\z{n}{\phi}{h-1} &= \z{n}{\varphi_2}{h-1}\; \lor\\ 
&\quad\left( \z{n}{\varphi_1}{h-1} \land \left( \bigvee_{t=0}^{h-1} \left(\z{loop}{}{t} \land \tildez{n}{\phi}{t}\right) \right)\right), \\
\tildez{n}{\phi}{t} &= \z{n}{\varphi_2}{t} \lor \left( z^{\varphi_1,n}_t \land \tildez{n}{\phi}{t+1} \right), \quad  t\leq h-2,\\
\tildez{n}{\phi}{h-1} &= \z{n}{\varphi_2}{h-1},
\end{split}
\end{equation}
where $\tildez{n}{\phi}{t}$ are auxiliary binary variables. As shown in \cite{biere2linear}, not introducing auxiliary variables results in trivial satisfaction of the \emph{until} operator.
 
\subsection{Outer logic constraints} 

Similar to the inner logic, we proceed by transforming a cLTL+ formula into ILP constraints. Given a cLTL+ formula $\mu$ and a time horizon $h$, we create $h$ binary decision variables $\mathbf{y^{cLTL+}} = \{\y{\mu}{t}\}$, where $t \in \{0,1,\dots ,h-1\}$ and ILP constraints $ILP(\mu)$. While doing so, we ensure that $\y{\mu}{t} = 1$ if and only if $(\Pi,K^*),t \models \mu$ where $K^*$ is the globally synchronous collective execution. We remind the reader that since ILP constraints are created recursively, creating the constraints for formula $\mu$ will create the constraints for all the inner logic formulas appearing in $\mu$. We denote by $ILP(\mu)$ the set of all resulting constraints that encode the satisfaction of $\mu$, and by $\mathbf{(z, y)^{cLTL+}}$, the set of all variables created in this process.
 
We provide encodings only for counting propositions since the rest of the semantics are identical. Let $\mu = [\phi, m] \in AP \times \bn$ be a temporal counting proposition. Then

\begin{equation}\label{eq:cp}
m > \sum_{n=1}^N \z{n}{\phi}{t} - M \y{\mu}{t} \geq m-M,
\end{equation}
where $M$ is a sufficiently large positive number, in particular, $M \geq N + 1$. Note that when $\y{\mu}{t} = 1$, the inequality on the right reduces to $ \sum_{n=1}^N \z{n}{\phi}{t}\geq m$. Moreover, the inequality on the left is trivially satisfied since $M \geq N + 1$. Conversely, when $\y{\mu}{t} = 0$, the inequality on the right is trivially satisfied and the inequality on the left reduces to $\sum_{n=1}^N \z{n}{\phi}{t} < m$. Therefore, $\y{\mu}{t} =1 $ 
if and only if the number of robots that satisfy $\phi$ at time $t$ is greater than or equal to 
$m$. Conversely, ($\y{\mu}{t} =0 $) if and only if the number of robots that satisfy $\phi$ at time $t$ is less than
$m$. Therefore, the ILP constraints in \eqref{eq:cp} are correct and consistent with the semantics of cLTL+.

\subsection{Overall optimization problem and its analysis}

The following optimization problem is formed to generate a solution to an instance of Problem \ref{prob1} given a horizon length $h$:
\begin{equation}
\label{eq:main}
  \begin{split}
  \text{Find }\quad &\{\mathbf{w}_{n}\}, \mathbf{z^{loop}}, \mathbf{(z, y)^{cLTL+}}\\
  \text{s.t.} \quad&\eqref{eq:dynamics}, \eqref{eq:loop},ILP(\mu)\text{ and } \y{\mu}{0} = 1.
  \end{split}
\end{equation}

Next we analyze this solution approach. The following theorem shows that the solutions generated by \eqref{eq:main} are sound.

\begin{prop}\label{prop:sound}
  If the optimization problem in \eqref{eq:main} is feasible for a cLTL+ formula $\mu$, then a collection $\Pi=\{\traj{n}\}_{n\in[N]}$ of trajectories can be extracted from $\{\mathbf{w}_{n}\}$ such that $(\Pi,K^*) \models \mu$.
\end{prop}

\begin{proof}
  Constraint \eqref{eq:dynamics} guarantees that the collection $\Pi$ of trajectories generated from $\{\mathbf{w}_{n}\}$ are feasible, consistent with the initial conditions and with the system dynamics. Furthermore, \eqref{eq:loop} ensures that these solutions can be extended to infinite trajectories of the form $\traj{n} = \s{n}{0}\dots \s{n}{l-1}\left(\s{n}{l}\dots \s{n}{h-1}\right)^\omega$. The ILP encodings \eqref{eq:ap}-\eqref{eq:until} of LTL formulas are sound \cite{biere2linear}, and the same encodings are also used for cLTL+ formulas by replacing $\z{n}{\phi}{t}$ with $\y{\mu}{t}$, where $\mu$ is any cLTL+ formula. The only exception is that \eqref{eq:ap} is replaced with \eqref{eq:cp}, which we showed to be correct. Therefore, the constraint $\y{\mu}{0} = 1$ together with $ILP(\mu)$ guarantees that $(\Pi,K^*) \models \mu$. Thus, if \eqref{eq:main} is feasible, then the globally synchronous execution of $\Pi$ solves Problem \ref{prob1}.
\end{proof}

As a corollary, it is easy to show that stutter invariance of formulas (see Theorem 7.92 from \cite{baier}) allows the generalization of the soundness result from globally synchronous executions to all synchronous executions:

\begin{corr}
If $\mu$ does not contain any next operator $\Next$, neither in the inner nor in the outer logic, then $(\Pi, K) \models \mu$ for all synchronous executions $K \in \mathcal{K}_N(0)$.
\end{corr}

The following theorem shows that encodings presented in \eqref{eq:dynamics}-\eqref{eq:main} are complete:
\begin{prop}\label{prop:sync_comp}
  If there is a solution to Problem \ref{prob1}, then there exists a finite $h$ such that \eqref{eq:main} is feasible.
\end{prop}

\begin{proof}
  In order to show that prefix-suffix form solutions are complete, we reduce Problem \ref{prob1} to a regular LTL control synthesis problem, for which prefix-suffix solutions have been shown to be complete \cite{baier}.

  Let $\Phi$ be the set of all inner logic formulas defined according to \eqref{eq:inner} over $AP$. Given any cLTL+ formula $\mu$, one can define an equivalent LTL formula over a new set of atomic propositions $AP' =\bigcup_{a \in AP} \{a_1,a_2,\dots a_N\}$. For each temporal counting proposition $tcp = [\phi,m]$ in $\mu$, we define a new set $\{\phi_1,\phi_2,\dots \phi_N \}$ of LTL formulas over $AP'$, where $\phi_n$ is obtained by replacing every atomic proposition $a \in AP$ with the corresponding $a_n \in AP'$. We then define $tcp' \doteq \bigvee_{i=1}^{I} (\bigwedge_{j \in J_i} \phi_{j})$, where $J = \{J_1,\dots,J_I\}$ is the set of all $m$-element subsets of $[N]$, hence $I = {N \choose m}$. Note that, $tcp'$ is equivalent to $tcp$, meaning that any collective execution that satisfy one will also satisfy the other. Even though this method increases the number of atomic propositions linearly and the length of the formula combinatorially with the number of robots, it will transform a cLTL+ formula into a regular LTL formula over a finite set of atomic propositions. 

  Next we create a product transition system $T' \doteq \Pi_n \bigt{n}$ with the set $AP'$ as its atomic propositions. Now Problem \ref{prob1} is reduced to a standard LTL synthesis problem and it can be solved using a model-checker to generate a prefix-suffix solution or to declare the non-existence of solutions (see e.g., \cite{belta2017formal}).
\end{proof} 

\begin{rem} The proof of Theorem \ref{prop:sync_comp} highlights the advantages of using cLTL+ in scenarios where robot identity is not critical for accomplishing the collective task. Although the problem can be reduced to a standard LTL synthesis problem as the proof suggests, the reduction results in a synthesis problem on a product transition system with size exponential in the number of robots, and with an LTL formula that is combinatorially longer than the cLTL+ formula. Indeed, without a convenient logic, just writing down that LTL formula would be a tedious and error-prone task.
\end{rem}

A few remarks on the complexity are in order. An instance of \eqref{eq:main} has $\mathcal{O}(hN(|S_n| + |\mu|))$ decision variables and constraints where $h$ is the solution horizon, $N$ is the number of robots, $|S_n|$ is the number of states of the largest transition system and $|\mu|$ is the length of the cLTL+ formula $\mu$. Enforcing collision avoidance introduces $\mathcal{O}(hN^2|S_n|)$ additional constraints. 

\subsection{cLTL encodings}\label{chp:solcLTL}
 Given an instance of Problem \ref{prob1}, if the specification $\mu$ can be expressed in cLTL and all robots have identical dynamics, more efficient encodings could be defined. In the following, we first define the problem where cLTL encodings could be used and then provide the corresponding encodings:
 
 \begin{Problem}\label{prob:cLTL}
  Given $N$ robots with identical dynamics $T = (S,\ra, AP, L)$, initial conditions $\{\s{n}{0}\}$, and a cLTL formula $\mu$ over $AP$, synthesize a collection $\Pi = \{\traj{1},\dots,\traj{N} \}$ of trajectories such that the globally synchronous collective execution of $\Pi$ satisfies $\mu$, i.e., $(\Pi,K^*) \models\mu$.
 \end{Problem}
 
Let the set $S$ of states be enumerated such that $S  = \{v^1,v^2,\dots,v^{|S|}\}$. Instead of individually encoding the dynamics of each robot, we define an \emph{aggregate state} vector $\mathbf{w} = [w^1, w^2,\dots w^{|S|}]^T$ where the $i^{th}$ row of $\mathbf{w}$ denotes the number of robots at state $v^i$. Similarly, the \emph{aggregate input} is defined as a vector $\mathbf{u} = [u_1^1, u_1^2, \dots, u_{1}^{|S|}, u_{2}^{1},\dots u_{2}^{|S|}, \dots u_{|S|}^{|S|}]^T$ where $u_i^j$ denotes the number of robots that transition from state $v^i$ to $v^j$. Note that the aggregate input is state-dependent since the total number of robots sent from a particular state to others cannot be greater than the number of robots in that state. Furthermore, the number of robots sent from a state can only be a non-negative integer. An input satisfying these conditions is called \emph{admissible} and $\Upsilon(\mathbf{w})$ denotes the set of all admissible inputs for a given state $\mathbf{w}$. The set $\Upsilon(\mathbf{w})$ can be captured by the following set of equalities:

\begin{equation}
\label{eq:total_input_constr}
\Upsilon(\mathbf{w}) = \left\{ \{ u^j_i\} : \begin{aligned}
& \sum_{j=1}^{|S|} u^j_i = w^i\\
& u^j_i = 0 \text{ if } (v^i, v^j) \not\in \rightarrow \\
& u_i^j \in \bn 
\end{aligned}  \right\}.
\end{equation}

The evolution of aggregate state can be captured by the following linear equalities:  
  \begin{equation}\label{eq:cLTL_dyn}
  \begin{split}
  \mathbf{w}(t+1) &= B\mathbf{u}(t),\\
  \mathbf{u}(t) &\in \Upsilon(\mathbf{w}(t)),
  \end{split}
  \end{equation}
  where $B$ is defined as $B \doteq I_{|S|}\otimes \mathbf{1}_{|S|}^T$ where $I_{|S|}$ is the identity matrix of size $|S|$ and $\otimes$ is the Kronecker product. 
  
Loop constraints for aggregate states can be written as:

\begin{equation}
  \begin{split}\label{eq:loop2}
  \mathbf{w}(h) &\leq  \:\mathbf{w}(t) + \mathbf{1}(1-z^{loop}_{t}), \\
  \mathbf{w}(h) &\geq \: \mathbf{w}(t) - \mathbf{1}(1-z^{loop}_{t}).
  \end{split}
\end{equation} 

Inner logic constraints are no longer needed since the cLTL inner logic is constrained to the grammar $\phi::=a$ where $a\in AP$.
  In the outer logic, only the encoding of temporal counting propositions in \eqref{eq:cp} needs modification.
  Let $\mu = [a,m]$ be a $tcp_{cLTL}$ and $S = \{ v^1, \dots, v^{|S|}\}$ be the set of states. We define the vector $\mathbf{v}^a \in \{0,1\}^{|S|}$ similar to \eqref{eq:ap}, that is, the $i^{th}$ entry of $\mathbf{v}^a $ is $1$ if and only if $a\in L(v^i)$. Then, for all $t = 0,\dots,h$, the constraints
  \begin{equation}\begin{split}\label{eq:cLTL}
  \mathbf{v}^a\mathbf{w}(t) &\geq m - M(1- y^{\mu}_t)\\
  \mathbf{v}^a\mathbf{w}(t) &\leq m + M y^{\mu}_t
  \end{split}\end{equation}
  ensure that $\y{\mu}{t}=1$ if and only if the number of robots that satisfy $a$ is greater than or equal to $m$. 
  The rest of the outer logic encodings are not modified and used as before. 
  
 Given a time horizon $h$, the following optimization problem is formed to generate solutions to an instance of Problem \ref{prob:cLTL}:
  
  \begin{equation}
  \label{eq:cLTL_ILP}
  \begin{split}
  \text{Find }\quad & \mathbf{u}(0),\dots,\mathbf{u}(h-1), \mathbf{z^{loop}}, \mathbf{ y^{cLTL+}}\\
  \text{s.t.} \quad&\eqref{eq:total_input_constr},\eqref{eq:cLTL_dyn}, \eqref{eq:loop2},ILP(\mu)\text{ and } \y{\mu}{0} = 1.
  \end{split}
  \end{equation}
  
  We now show how a solution of \eqref{eq:cLTL_ILP} can be mapped to a collection $\{\traj{n}\}$ of individual trajectories. Given initial conditions $\s{n}{0}$, and $\mathbf{u}(0)$, randomly choose $u^j_i$ robots from state $v^i$ and assign their next state as $v^j$. This is always possible since $\mathbf{w}(0)$ is well defined and $\mathbf{u}(0) \in \Upsilon(\mathbf{w}(0))$. Continuing in this manner, we can generate the collection $\{\traj{n}\}$ whose globally synchronous collective execution satisfies the specification $\mu$. Details of a similar constructions of individual trajectories can be found in \cite{nilsson2018control}.
  
  Before proceeding to the asynchronous problem, we remind the reader of two important things: (i) the ILP constraints in \eqref{eq:cLTL_ILP} are consistent with cLTL+ semantics, therefore soundness and completeness guarantees follow from Theorems \ref{prop:sound} and \ref{prop:sync_comp}. (ii)
  An instance of \eqref{eq:cLTL_ILP} has $\mathcal{O}(h(|\to| + |\mu|))$ decision variables and constraints where $|\to|$ is the number of transitions and $|\mu|$ is the length of the formula. Crucially, the number of decision variables and constraints
  does not depend on the number of robots. Therefore, it easily scales to very large number of robots as demonstrated in Section \ref{chp:results}.

\section{Robustness to asynchrony}

Incorporating a concept of time-robustness into our algorithm is useful since it is difficult to perfectly synchronize the motion of robots in real-life applications. This section presents small modifications to the original algorithm that allow one to synthesize trajectories that are robust to bounded synchronization errors.

Synchronous execution assumes that multiple robots can transition from one discrete state to another at the same time. However, this is not always possible in reality where robots may move slower or faster than intended, leading to asynchronous switching times as illustrated in Figure \ref{fig:running_example}. To exemplify, consider a task that requires multiple robots to satisfy a certain proposition at the same time. Let $\mu = \eventually[\phi, m]$ be a $tcp$, $\Pi$ be a collection of trajectories and $K$ be a synchronous collective execution. Assume that $[\phi, m]$ holds for a single time step $t$ and fails to hold for all others, i.e., $(\Pi,K),t \models [\phi, m]$ for some $t$ and $(\Pi,K),t' \not\models [\phi, m]$ for all $t'\neq t$. While such a $\Pi$ satisfies $\mu$ for the synchronous execution it is not always a desirable collection, because if $K$ becomes asynchronous due to one of the robots moving slower than intended, correctness guarantees would no longer be valid and $\mu$ would not be satisfied. This fact motivates us to generate solutions that are robust to such asynchrony. 

For most non-trivial specifications however, finding a collection of trajectories that is robust to unbounded asynchrony would be challenging if not impossible. If, however, an upper bound on the asynchrony is known, one can generate robust solutions such that satisfaction of the task is guaranteed even under the worst-case scenario.

To reason about asynchronicity we define the concept of \emph{anchor time} for collective executions.

\begin{defn}
  For a given collective execution $K=[k_1\dots k_N]$, the \textbf{anchor time mapping} $b_K$ maps the time index $t$ to the smallest local counter value $\ktn{n}{t}$, i.e., $b_K(t) = \min_{n} \ktn{n}{t}$.
\end{defn}

For a $\tau$-bounded collective execution $K \in \mathcal{K}_N(\tau)$ and a given time step $t$, at least one local counter has the value $b_K(t)$ and all other local counters are limited to an interval: $k_n(t)\in [b_K(t), b_K(t) + \tau]$ for all $n$. For the globally synchronous collective execution $K^*$, the anchor time mapping is the identity mapping on $\mathbb{N}$. In Figure~\ref{fig:running_example}, ``anchoring robots'' at each time step are highlighted with a black circle and anchor times are written in bold.

Having defined the ``anchor time'', we now formally define the concept of robust satisfaction for a collection of trajectories.

\begin{defn}
  \label{def:robust_satisfaction}
  A collection of trajectories $\Pi = \{ \traj{1}, \dots, \traj{N} \}$ \textbf{$\boldsymbol{\tau}$-robustly satisfies} $\mu$ at time $t$, denoted
  \begin{equation}
  \Pi,t  \models_\tau \mu,
  \end{equation}
  if and only if for all $K \in \mathcal{K}_N(\tau)$ and for all $T \in b_K^{-1}(t)$,
  \begin{equation}
  (\Pi,K),T \models \mu.
  \end{equation}
\end{defn}

In other words, a specification $\mu$ is $\tau$-robustly satisfied at time $t$ by $\Pi$ if every $\tau$-bounded collective execution $K$ of $\Pi$ satisfies $\mu$ at all time instances $T$ for which the anchor time is $t$. Consider the set of trajectories $\Pi = \{ \traj{1}, \traj{2}, \traj{3} \}$ and an asynchronous collective execution $K$ given in Example~\ref{ex:running_example}. For $\Pi, 1 \models_\tau \mu$ to hold; we must have $(\Pi,K), T \models \mu$, for all $T \in \{ 1,2,3\}$ since $b_K^{-1}(1) = \{1,2,3\}$. Additionally, the same argument must hold for every possible $K' \in \mathcal{K}_N(\tau)$. If $\Pi,0\models_\tau \mu$, we say that the collection $\Pi$ satisfies cLTL+ formula $\mu$ and write $\Pi\models_\tau \mu$ for short.

Before presenting modified encodings that incorporate robustness to asynchrony, we remind the reader that the robots are allowed to stutter as indicated by Definition \ref{local_counter}. Any inner logic formula containing `$\Next$' can always be violated by a single robot when robots are allowed to stutter. Hence, we restrict attention to the case where inner logic formulas are in LTL$_{\setminus\Next}$. We further assume that a cLTL+ formula is given in positive normal form (PNF) according to the following syntax:
\begin{equation}
\label{eq:logic_restr}
\begin{aligned}
    \mu ::= & True \mid tcp  \mid \mu_1 \land \mu_2 \mid \mu_1 \lor \mu_2, \\
            & \mid \Next \mu 
            \mid \mu_1 \until \mu_2 \mid \mu_1 \release \mu_2.
\end{aligned}
\end{equation}

\begin{rem}
The negation operator can be omitted without loss of generality for two reasons. First, any LTL formula can be transformed into positive normal form (PNF) \cite{baier}, where the negation operator appears only before atomic propositions. Since the syntax of cLTL+ is identical to LTL, hence any cLTL+ formula can also be written in PNF where negation only appears before $tcp$'s. Second, given an arbitrary temporal counting proposition $\mu = [\phi, m]$, the statement $\neg \mu$ can be replaced by $\mu' = [\neg\phi, N+1-m]$. Clearly, if there are at least $N+1-m$ robots satisfying $\neg\phi$, then $\phi$ is satisfied by less than $m$ robots; hence, $\mu \equiv \mu'$. Thus, the omission of the negation operator is without loss of generality.
\end{rem}

Finally, we formally define the robust version of Problem \ref{prob1} as follows:

\begin{Problem}\label{prob1_robust}
  Given $N$ robots with dynamics $\{\bigt{n} = (S_n,\ra_n, AP, L_n)\}$, initial conditions $\{\s{n}{0}\}$, a cLTL+ formula $\mu$ given in PNF over LTL$_{\setminus{\Next}}$, and an upper bound on the asynchrony $\tau$, synthesize a collection $\Pi = \{\traj{1}, \dots,\traj{N} \}$ of trajectories $\traj{n}$ that $\tau$-robustly satisfies $\mu$, i.e., $\Pi \models_\tau \mu$.
\end{Problem}

We propose slight modifications to the encodings presented in Section \ref{chp:sol} to generate a collection of trajectories that are $\tau$-robust. Firstly, we define $\tau$ new Boolean vectors $\w{n}{h+1},\w{n}{h+2} \dots \w{n}{h+\tau}$ to represent the state of robot $n$ ``after the loop'' such that $\w{n}{h+k} =\w{n}{l+k}$ for some $l<h$ and $k = 0,1,\dots, \tau$. Secondly, given temporal counting proposition $\mu = [\phi, m]$, we introduce a new decision variable $\zr{n}{\phi}{t}$ for each $\z{n}{\phi}{t}$:

\begin{align}
\label{eq:ap_robust{t}}
\zr{n}{\phi}{t} &= \bigwedge_{k=0}^\tau \z{n}{\phi}{t+k}, \qquad \qquad \text{for } 0\leq t<h.
\end{align}
 
Note that, $\z{n}{\phi}{t}$ is defined for all $t\leq h+\tau$ due to newly defined additional state vectors. These new variables $\zr{n}{\phi}{t}$ can be seen as the robust versions of $\z{n}{\phi}{t}$. In order for $\zr{n}{\phi}{t}=1$ to hold, robot $n$ needs to satisfy the inner logic formula $\phi$ not only at time step $t$, but also for the next $\tau$ steps. Since at anchor time $t$, the local times are bounded as $t \leq \ktn{n}{t} <t+\tau$, this robustification ensures that  robot $\mathcal{R}_n$ satisfies $\phi$ at anchor time $t$, regardless of the asynchrony. 

We now define the modified outer logic constraints. As before, these constraints are constructed recursively. Let $\mu = [\phi, m]$ be a $tcp$ such that $m>1$. Then \eqref{eq:cp} is modified as
\begin{equation}\label{eq:cp_robust}
m>\sum_{n=1}^N \zr{n}{\phi}{t} - M \y{\mu}{t} \geq m-M.
\end{equation}
For the special case where $\mu = [\phi, 1]$, we use 
\begin{equation}\label{eq:cp_robust1}
\begin{split}
1 &> \sum_{n=1}^N \zr{n}{\phi}{t} - M \tildey{\mu}{t} \geq 1-M,\\
N &> \sum_{n=1}^N \z{n}{\phi}{t} - M \bary{\mu}{t} \geq N-M,\\
\y{\mu}{t} &= \tildey{\mu}{t}\lor \bary{\mu}{t}.
\end{split}
\end{equation}

In the synchronous setting, satisfying a temporal counting proposition $\mu$ only for an instant would be enough. However, this is not desirable since robots might not be perfectly synchronized. Equations \eqref{eq:cp_robust} and \eqref{eq:cp_robust1} ensures that all $\tau$-bounded executions satisfy $\mu$ at all time instances with anchor time $t$, by replacing each $\z{n}{\phi}{t}$ with its robust counterpart $\zr{n}{\phi}{t}$. As a result, even in the worst case of asynchrony, there would be an instant where $\mu$ is satisfied.

Encodings of some of the outer level operators are also modified slightly. For conjunction and next operators, no modification is needed: if $\mu = \mu_1\land \mu_2$ and $\eta = \bigcirc \mu$ where each $\mu_i$ is a cLTL+ formula in PNF form, then $\y{\mu}{t}  = \y{\mu_1}{t} \land \y{\mu_2}{t}$ and $\y\eta{t} = \y\mu{t+1}$.

Disjunction is encoded in two different ways: If all operands are temporal counting propositions, i.e, $\mu = \bigvee_i \mu_i$ where $\mu_i = [\phi_i,m_i]$, then
\begin{equation}\label{eq:dis_robust}
\y{\mu}{t}  = \bigvee_i \y{\mu_i}{t} \lor \left(\sum_{n=1}^N \zr{n}{(\bigvee_i\phi_i)}{t} > \sum_{i} (m_i-1)\right)
\end{equation}
is used.
Note that $\zr{n}{(\bigvee_i\phi_i)}{t}$ is only defined if all $\mu_i$ are $tcp$. In all other cases, we use the standard encoding:
\begin{equation}\label{eq:dis_robust2}
\y{\mu}{t}  = \bigvee_i \y{\mu_i}{t}.
\end{equation}

If the disjunction contains both $tcp$s and other formulas, then it can be re-written to leverage the less conservative encodings in \eqref{eq:dis_robust}.
The motivation behind \eqref{eq:dis_robust} is that, a collection $\{\traj{n}\}$ might not $\tau$-robustly satisfy neither $\mu_1$ or $\mu_2$ but can still $\tau$-robustly satisfy $\mu_1 \lor\mu_2$ as demonstrated by the following example:
\begin{example}\label{ex:dis}
  Let $\mu = \mu_1 \lor \mu_2 = [\phi_1,2] \lor [\phi_2,2]$ be a cLTL+ formula and let a collection $\Pi =\{\traj{1},\traj{2},\traj{3}\}$ be given with the following traces:
  \begin{align*}
  \sigma(\traj{1})&=\{\phi_1\} \;\{\phi_1\}\;\{\phi_1\}\;\dots\\
  \sigma(\traj{2})&=\{\phi_1\} \;\{\phi_2\}\;\{\phi_2\}\;\dots\\  
  \sigma(\traj{3})&=\{\phi_2\} \;\{\phi_2\}\;\{\phi_2\}\;\dots
  \end{align*}
  If $\tau=1$, the collection $\Pi$ does not robustly satisfy neither $\mu_1$ nor $\mu_2$ at anchor time $0$. On the other hand, for all time steps with anchor time $t$, any arbitrary $\tau$-bounded asynchronous execution satisfies either $\mu_1$ or $\mu_2$. This implies that $\Pi\models_\tau \mu.$
\end{example}  

Equation \eqref{eq:dis_robust} limits the number of robots who neither satisfy $\phi_1$ nor $\phi_2$ at anchor time $t$. By doing so, it ensures that either $\mu_1$ or $\mu_2$ is satisfied by the collection. Observe that \eqref{eq:dis_robust} reduces to standard encodings for $\tau=0$.

Due to changes in the outer disjunction encodings, the outer ``until'' operator needs to be modified as well. Let $\eta = \mu_1 \until \mu_2$ where $\mu_i$ is a cLTL+ formula for $i=1,2$. Then 
\begin{equation}\label{eq:until_robust}
\begin{split}
\y{\eta}{t} &= \y{\mu_1\lor\mu_2}{t} \land \left( \y{\mu_2}{t} \lor \y{\eta}{t+1}  \right), \quad t\leq h-2,\\
\y{\eta}{h-1} &= \y{\mu_1\lor \mu_2}{h-1} \land\\
& \quad \left( \y{\mu_2}{h-1} \lor \left(\bigvee_{t=0}^{h-1} \left(\z{loop}{}{t} \land \tildey{\eta}{t} \right)\right) \right), \\
\tildey{\eta}{t}  &= \y{\mu_1\lor\mu_2}{t} \land \left( \y{\mu_2}{t} \lor \tildey{\eta}{t+1} \right), \quad  t\leq h-2,\\
\tildey{\eta}{h-1} &= \y{\mu_2}{h-1}.
\end{split}
\end{equation}

If $\mu_2$ is $\tau$-robustly satisfied at time $t$, then $\eta$ is $\tau$-robustly satisfied at time $t$, by definition of `until'. In this case both $\y{\mu_2}{t} = 1$ and $\y{\mu_1\lor \mu_2}{t}=1$ would hold, hence $\y{\eta}{t}$ would evaluate to $1$, as expected. If $\mu_2$ is \emph{not} $\tau$-robustly satisfied at time $t$, \eqref{eq:until_robust} enforces $\eta$ and $\mu_1 \lor \mu_2$ (instead of $\mu_1$ as in \eqref{eq:until}) to be $\tau$-robustly satisfied at anchor times $t+1$ and $t$, respectively. This again guarantees that $\eta$ is  $\tau$-robustly satisfied at anchor time $t$. Auxiliary variables are used again to ensure $\mu_2$ is satisfied at some point. As before, \eqref{eq:until_robust} reduces to the standard until encodings when $\tau=0$. 

Furthermore, we provide the encodings for the ``release'' operator, which is identical to the standard encodings used in the literature: if $\eta = \mu_1 \release \mu_2$, then
\begin{equation}\label{eq:release}
\begin{split}
\y{\eta}{t} &= \y{\mu_2}{t} \land \left( \y{\mu_1}{t} \lor \y{\eta}{t+1}  \right), \quad t\leq h-2,\\
\y{\eta}{h-1} &= \y{\mu_2}{h-1} \land\\
&\quad \left( \y{\mu_1}{h-1} \lor\left(\bigvee_{t=0}^{h-1} \left(\z{loop}{}{t} \land \tildey{\eta}{t} \right)\right) \right), \\
\tildey{\eta}{t}  &= \y{\mu_2}{t} \land \left( \y{\mu_1}{t} \lor \tildey{\eta}{t+1} \right), \quad  t\leq h-2,\\
\tildey{\eta}{h-1} &= \y{\mu_2}{h-1}.
\end{split}
\end{equation}

Release encodings guarantees that if $\mu_1$ is $\tau$-robustly satisfied for all anchor times $t$, then $\mu_2$ is $\tau$-robustly satisfied for all times up to and including $t$. The key difference from the until operator is that $\mu_1$ does not have to be satisfied at all if $\mu_2$ is satisfied for all times.

Given an instance of Problem \ref{prob1_robust} and a horizon length $h$, let $ILP_{\tau}(\mu)$ be the set of ILP constraints and $\mathbf{(z, r, y)^{cLTL+}}$ the decision variables created by using the robust encodings \eqref{eq:cp_robust}-\eqref{eq:release}. We obtain the robust solution by solving the following optimization problem:
\begin{equation}\label{eq:main_robust}
\begin{split}
\text{Find }\quad &\{\mathbf{w}_n\}, \mathbf{z^{loop}}, \mathbf{(z, r, y)^{cLTL+}}\\
\text{s.t.} \quad&\eqref{eq:dynamics}, \eqref{eq:loop},ILP_{\tau}(\mu)\text{ and } \y\mu 0 = 1.
\end{split}
\end{equation}

The following theorems show that the solution method proposed for the asynchronous case is sound, and also complete under certain conditions. The proofs are provided in the Appendix.

\begin{prop}
\label{prop:robustness_is_sound}
    If the optimization problem in \eqref{eq:main_robust} is feasible for a cLTL+ formula $\mu$ given in PNF over LTL$_{\setminus\bigcirc}$, then a collection $\Pi = \{ \traj{1}, \dots, \traj{N} \}$ of trajectories can be extracted such that $\Pi \models_\tau \mu$. That is, the modified encodings in \eqref{eq:ap_robust{t}}-\eqref{eq:release} are sound.
\end{prop}

As shown in Example \ref{ex:dis}, the disjunction operator introduces some conservatism. Furthermore, the disjunction operation is used in the encodings of ``until'' and ``release''. Therefore, completeness results from Section IV are no longer valid in the asynchronous setting. The next result clarifies the conditions when the robust encodings are complete:

\begin{prop}
\label{prop:robustness_is_complete}
  Given a cLTL+ formula $\mu$ given in PNF over LTL$_{\setminus\bigcirc}$, if all of the following hold, then there exists a finite $h$ such that \eqref{eq:main_robust} has a solution (i.e., the modified encodings are complete).
  \begin{itemize}
    \item there exists a collection $\Pi = \{ \traj{1}, \dots, \traj{N} \}$ of trajectories in prefix-suffix form that $\tau$-robustly satisfies $\mu$, i.e., $\Pi\models_\tau \mu$,
    \item $AP$ is a set of mutually exclusive atomic propositions, i.e., for all $\phi_1,\phi_2 \in AP$; $\phi_1\land\phi_2= False$,
    \item the specification $\mu$ over $AP$ is on the form 
    \begin{equation}
    \label{eq:complete}
    \mu = True \mid tcp \mid \mu_1 \land \mu_2 \mid tcp_1 
    \lor tcp_2 \mid  tcp_1 \until tcp_2 \mid \bigcirc \mu
    \end{equation}
  where $tcp,tcp_1,tcp_2\in AP\times \Phi$ and $\mu,\mu_1,\mu_2$ are obtained according to \eqref{eq:complete}.
  \end{itemize}
\end{prop}

The commonly used ``$\eventually (eventually)$'' operator can also be defined without losing completeness: $\eventually [\phi,m]\doteq [\neg\phi,N-m+1]\until [\phi,m]$. In most real world applications, several tasks are required to be completed in conjunction, which can be expressed as in \eqref{eq:complete}. Furthermore, many interesting specifications including safety $(\always)$, liveness ($\always\eventually$), etc., can be captured in the form of \eqref{eq:complete} for a given time horizon $h$. For example, safety specifications can be encoded as $\always [\phi,m] = [\phi,m] \land \bigcirc[\phi,m] \land\dots \land \bigcirc^{h-1} [\phi,m]$\footnote{The notation $\bigcirc^{h-1}$ corresponds to $(h-1)$ concatenated $\bigcirc$ operators}. 

\begin{rem}
The alternative solution method proposed in Section \ref{chp:solcLTL} uses more efficient encodings when the specifications are given in $cLTL$. However, these encodings use aggregate dynamics, therefore it is not possible to keep track of identities of the robots during synthesis. Hence, robust solutions cannot be generated with this alternative method.
\end{rem}

Robustifying the trajectories increases the complexity as a function of $\tau$. In particular, an instance of \eqref{eq:main_robust} has $\mathcal{O}(\tau N(|S_n| + h|\mu|))$ additional decision variables and $\mathcal{O}(\tau N^2 h|S_n|)$ additional constraints compared to \eqref{eq:main}. The effect of these additional variables and constraints on solution time is shown in Section VII.

\section{Extensions and Discussion}

In this section, we discuss two possible extensions of cLTL+. Firstly, we show how to handle continuous-state dynamics directly instead of transition systems. Secondly, we provide an extension of cLTL+ syntax that allows tasks to be assigned to specific robots or robot groups. 

\subsection{Extension to Continuous-State Dynamics}
\label{ssec:continuos}

Up to now, we assumed that robot dynamics are modeled by transition systems. Given continuous dynamics, discrete abstraction techniques could be used to obtain transition systems. However, abstraction computations are costly and do not scale well with the number of dimensions. This section provides slight modifications to the earlier encodings such that continuous-state discrete-time dynamics can be handled directly.

Assume that the robot dynamics are given as 
\beq \label{eq:dynamics_cont}
\w n {t+1} = f_n(\w n {t}, \uu n t),
\eeq
where $\w n {t} \in \br^{d_{w}}$ and $\uu n t \in \br^{d_{u}}$ denote the state and input of robot $n$ at time $t$, respectively. 

The first modification is to replace the constraints in \eqref{eq:dynamics} with \eqref{eq:dynamics_cont} for all $n \in [N]$ and for all $t$. The loop constraints in \eqref{eq:loop} are then modified as follows:

\begin{equation}
  \begin{split}\label{eq:loop_cont}
  \w n {h} \leq & \w n {t} + M(1-\z{loop}{}{t}), \\
  \w n {h} \geq & \w n {t} - M(1-\z{loop}{}{t}),
  \end{split}
\end{equation}
where $M$ is a sufficiently large number. Equation \eqref{eq:loop_cont} enforces a loop by constraining $\w n {h}$ to be equal to $\w n {t}$ for some $t$.

Next, we modify \eqref{eq:ap} to accommodate continuous states. We assume that each atomic proposition $a \in AP$ corresponds to a convex polytope $\{ w \in \br^{d_w} \mid H^{a} w \leq h^{a} \}$, where $H^{a} \in \br^{d_a\times d_w}
$ and $h^{a} \in \br^{d_a}
$. Then for each atomic proposition and for all $t$ and $n \in [N]$, we replace the inequality constraints in \eqref{eq:ap} with the following:
\begin{subequations}\label{eq:ap_cont}
  \begin{align}\label{eq:ap_cont1}
  H^{a} \w n {t} &\leq h^{a} + M(1 - e_n^{a}(t)), \\
  \label{eq:ap_cont2}
  H^{a} \w n {t} &\geq h^{a} + \epsilon - M e_n^{a}(t), \\
  \label{eq:ap_cont3}
  \z{n}{a}{t} &\leq e_n^{a,(i)}(t) \qquad \text{for } i=1,\dots,d_a,\\
  \label{eq:ap_cont4}
  \z{n}{a}{t} &\geq 1- d_a + \mathbf{1} e_n^{a}(t),
  \end{align}
\end{subequations}

where $\epsilon$ is an infinitesimally small and $M$ is a sufficiently large number, and $e_n^a$ is a binary vector of size $d_a$. The $i^{th}$ row of $e_n^a$ is denoted by $e_n^{a,(i)}(t)$ and is used to check the satisfaction of the $i^{th}$ linear constraint. In equations \eqref{eq:ap_cont1} and \eqref{eq:ap_cont2}, the $i^{th}$ linear constraint is satisfied if and only if $e_n^{a,(i)}(t)=1$. Furthermore, with equations \eqref{eq:ap_cont3} and \eqref{eq:ap_cont4}, we ensure that $\w{n}{t} \in \{ w \in \br^{d_w} \mid H^{a} w \leq h^{a} \}$ if and only if $\z{n}{a}{t} =1$. This result is identical to \eqref{eq:ap}; thus, no other modifications are needed to use $\z{n}{a}{t}$ in \eqref{eq:neg}-\eqref{eq:cp}.

Finally, we modify the optimization problem to account for auxiliary variables. Let $\mathbf{e}^{cLTL+}$ denote the set of all auxiliary variables created by \eqref{eq:ap_cont}. We form the following optimization problem to find solutions:

\begin{equation}
\label{eq:main_cont}
\begin{split}
\text{Find }\quad &\{\uu{n}{0}\dots \uu{n}{h-1}\}, \mathbf{z^{loop}}, \mathbf{(e, z, y)^{cLTL+}}\\
\text{s.t.} \quad&\eqref{eq:dynamics_cont}, \eqref{eq:loop_cont},ILP(\mu)\text{ and } \y \mu 0 = 1.
\end{split}
\end{equation}

\begin{rem} 
  Given initial condition $w^{n}_0$ and inputs $\{\uu{n}{0}\dots \uu{n}{h-1}\}$, state $\w n {t}$ can be found by \eqref{eq:dynamics_cont}. Hence, no decision variables are needed for the states.
\end{rem}

\todo[inline]{Maybe cite Vasus paper here too, or write more about continuous-state part in introduction}

\begin{rem}
  The resulting feasibility problem is a \emph{mixed integer linear program (MILP)} if linear continuous-state dynamics are used.
\end{rem}

As it is stated before, obtaining discrete abstractions from continuous dynamics is computationally expensive: the size of the transition system typically grows exponentially with the dimensionality of robot states. Since each discrete state in the transition system introduces a binary decision variable in the discrete-space formulation, the size of the optimization problem in \eqref{eq:main} can grow quickly. On the other hand, in \eqref{eq:main_cont}, each continuous state is represented with a single continuous decision variable. While the number of auxiliary binary decision variables introduced by \eqref{eq:ap_cont} depends on the specific problem instance, the continuous approach can be favorable when compared to an abstraction approach. 

\subsection{Extension of cLTL+ Syntax}

This section provides a straightforward extension of the cLTL+ syntax inspired by censusSTL proposed in \cite{xu2016census}. Up to now, the logic is oblivious as to which robot satisfies what atomic proposition, or task. In most multirobot systems, robots have heterogeneous capabilities and certain tasks can only be performed by a specific subset of robots. For example, imagine a collection of drones and a reconnaissance mission that includes, among other things, taking aerial photos of a region. If not all of the drones have cameras, one might want to identify those that can take photos and require subtasks that involve photography to be completed by this subset. Similarly, in a collective of robots where one robot is designated to be the leader it may be desirable to specify that the other robots periodically have to report to the leader.

To be able to specify such tasks, the temporal counting propositions ($tcp$) can be modified to contain the subset of robots that are designated with satisfying the inner logic formula. Redefine $tcp$ as a tuple consisting of an atomic proposition, a non-empty set of robots and a non-negative integer, i.e., $\mu = [\phi, \mathcal{S}, m] \in \Phi \times 2^{[N]} \times \bn$. Here satisfaction of $\mu$ at time $t$ requires at least $m$ robots from the subset $\mathcal{S} \in 2^{[N]}$ to satisfy $\phi$ at time $t$. By modifying $tcp$'s in this manner we can assign individual tasks to a specific subset of robots. To exemplify, given a collective $\mathcal{S}$ of drones, let $\mathcal{S}_c \in \mathcal{S}$ denote those with camera. Then the temporal counting proposition $tcp = [a , \mathcal{S}_c, m]$ would be satisfied if at least $m $ drones from $\mathcal{S}_c$ visit regions marked by $a\in AP$ to take aerial photos.

Let $\mu = [\phi, \mathcal{S}, m]$.
 We modify \eqref{eq:cp} as follows to account for the change in $tcp$ definition:

\begin{equation}
\label{eq:census}
m>\sum_{n\in \mathcal{S}} \z{n}{\phi}{t} - M \y {\mu} t 
\geq m-M.
\end{equation}

Similarly, for the robustness case, we modify \eqref{eq:cp_robust} as follows:

\begin{equation}\label{eq:cp_robust_ext}
m > 
\sum_{n\in \mathcal{S}} \zr{n}{\phi}{t} - M \y {\mu} t \geq m-M.
\end{equation}

It is straightforward to see that \eqref{eq:census} and \eqref{eq:cp_robust_ext} preserve all of the soundness and completeness guarantees for this extension.

\section{Results}
\label{chp:results}

This section demonstrates the proposed method on an emergency response and presents scalability results. All experiments are run on a laptop with 2.5 GHz Intel Core i7 and 16 GB RAM and Gurobi \cite{gurobi} is used as the underlying ILP solver. Our implementation can be accessed from \url{https://github.com/sahiny/cLTL-synth}.


\subsection{Emergency response example}

Assume $N=10$ robots are deployed in a workspace, which can be seen from Figure \ref{fig:arena}. The workspace is discretized into $10 \times 10$ cells and each robot is modeled with a transition system with $100$ states, each corresponding to a single cell. At each step, robots can either choose to stay put or travel to any of the four neighboring cells without leaving the workspace. We remark that a monolithic LTL solution for this problem would have required constructing a transition system with $100^{10}$ states.

\begin{figure}[t] 
  \centering 
  \includegraphics[width=.90\linewidth]{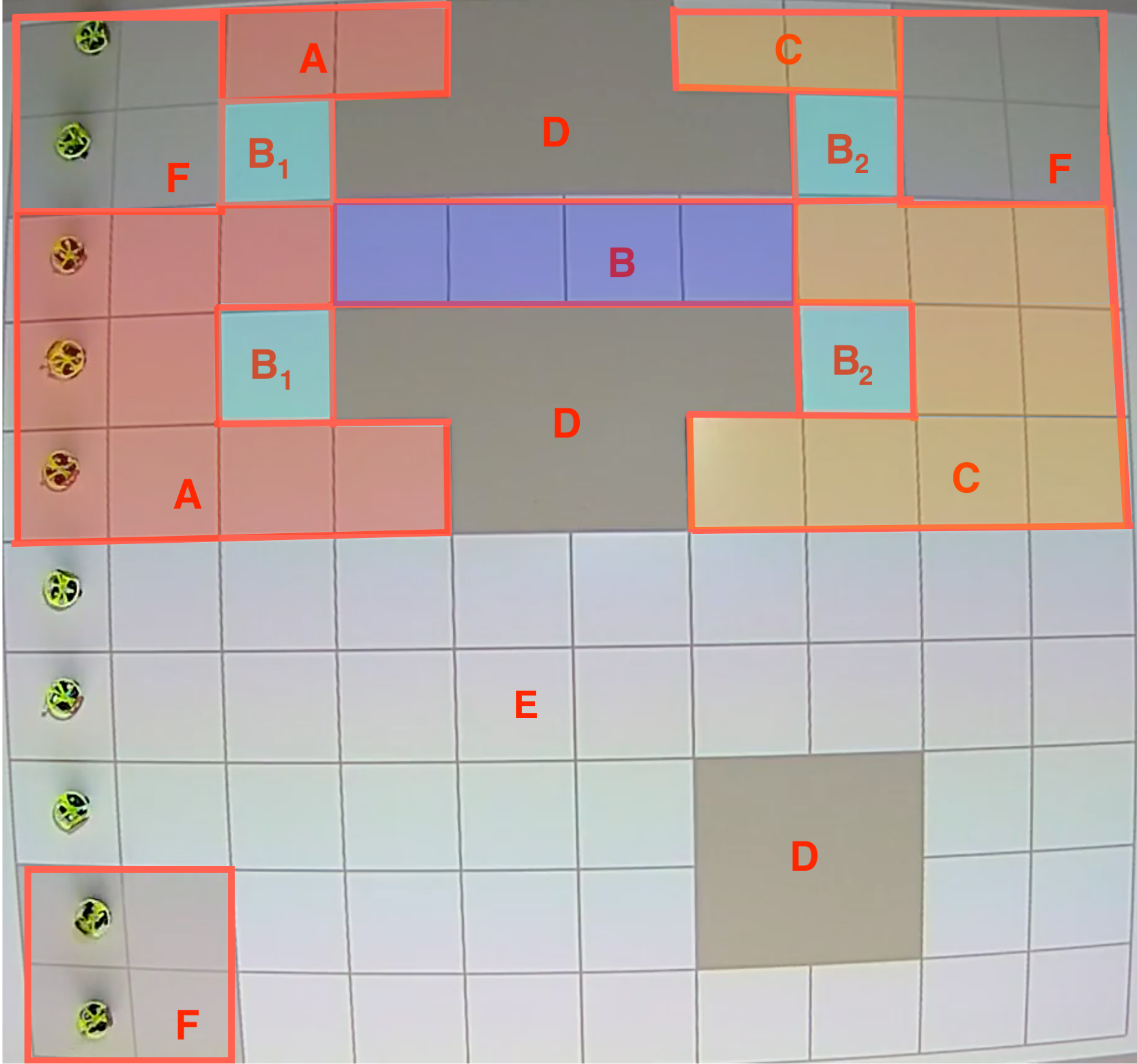}
  \caption{Workspace: $A$, $C$, and $E$ represent different neighborhoods, $B$ represents a fragile bridge, $F$ represents charging stations and $D$ represents inaccessible zones.}\label{fig:arena} 
\end{figure}

The specification is of the form $\mu=\bigwedge_{i=1}^8 \mu_i$, including: 
\begin{itemize}
  \item collision with obstacles, which are marked with $D$, should be avoided ($\mu_1 = \always\neg[D,1]$).
  
  \item the bridge, marked by $B$, must not be occupied by more than $2$ robots ($\mu_2 = \always\neg[ B,3]$).
    
  \item each robot should visit charging stations, marked by $F$, infinitely many times ($\mu_3 = [\always\eventually F, N]$).
  
  \item region $A$ and $C$ must be populated with at least half of the robots and should be left empty, infinitely many times ($\mu_4 = \always\eventually[A,N/2]$, $\mu_5 = \always\eventually[C,N/2]$, $\mu_6 = \always\eventually(\neg[ A,1])$ and $\mu_7 = \always\eventually(\neg [C,1]$)).
  
  \item bridge should be empty until it is inspected from both sides ($\mu_8 = (\neg[ B,1]) \until \left( [B_1,1] \land [B_2,1] \right)$).
\end{itemize}
In addition to these specifications, we require that robots avoid collisions with each other.
We posit a time horizon $h=35$ and solve the optimization problem \eqref{eq:main_robust} for the synchronous case $\tau=0$. 

Important frames obtained from the $\tau=0$ solution are shown in Fig.~\ref{fig:solution}. Even though all specifications are met by this solution for a synchronous execution, it could easily break with the introduction of asynchrony. For instance, note that region $A$ is emptied (resp. region $C$ is populated with more than $5$ robots) only for a single time step at $t=16$ (resp. $t=18$). Hence, a single-step delay of a single robot could result in violation of $\mu_6$ (resp. $\mu_5$). Similarly, a robot enters the bridge for the first time at $t=11$, which is the exact same time step when the bridge is inspected from both sides. If one of the robots inspecting the bridge moves slower than intended, $\mu_8$ would be violated.

To prevent such violations, we set $\tau=2$ and solve the resulting optimization problem. As it is shown in Fig.~\ref{fig:robust_solution}, this time the number of robots in $A$ (resp. in $C$) is greater than or equal to $5$, starting from $t=1$ until $t=3$ (resp. from $t=20$ until $t=22$). Furthermore, when the number of robots in region $A$ is greater than or equal to $5$, there are no robots in region $C$, and vice versa. Therefore, even in the worst case of bounded asynchrony, there will be at least one time instance where $A$ is populated with $5$ robots and another time instance where $A$ is empty. The same arguments hold for region $C$, as well. Additionally, the robots are more careful when crossing and the bridge: the bridge is first inspected at $t=11$ and no robots enter the bridge until $t=13$. Thus, the specification $\mu$ is satisfied even in the worst case of asynchrony.

We have implemented the trajectories extracted from the robust solution on real ground robots in Robotarium \cite{pickem2017robotarium}. In this experiment, robots track their respective trajectories using feedback from a top-mounted camera, and do not communicate with each other during runtime. The asynchrony is limited to $2$ discrete transitions. The video of the experiment can be viewed from \url{https://youtu.be/u8G-ewEEO6E}. As can be seen in the video, robots satisfy their tasks and avoid collisions despite the asynchrony.

\begin{figure*}[t] 
  \centering 
  \includegraphics[width=.92\linewidth]{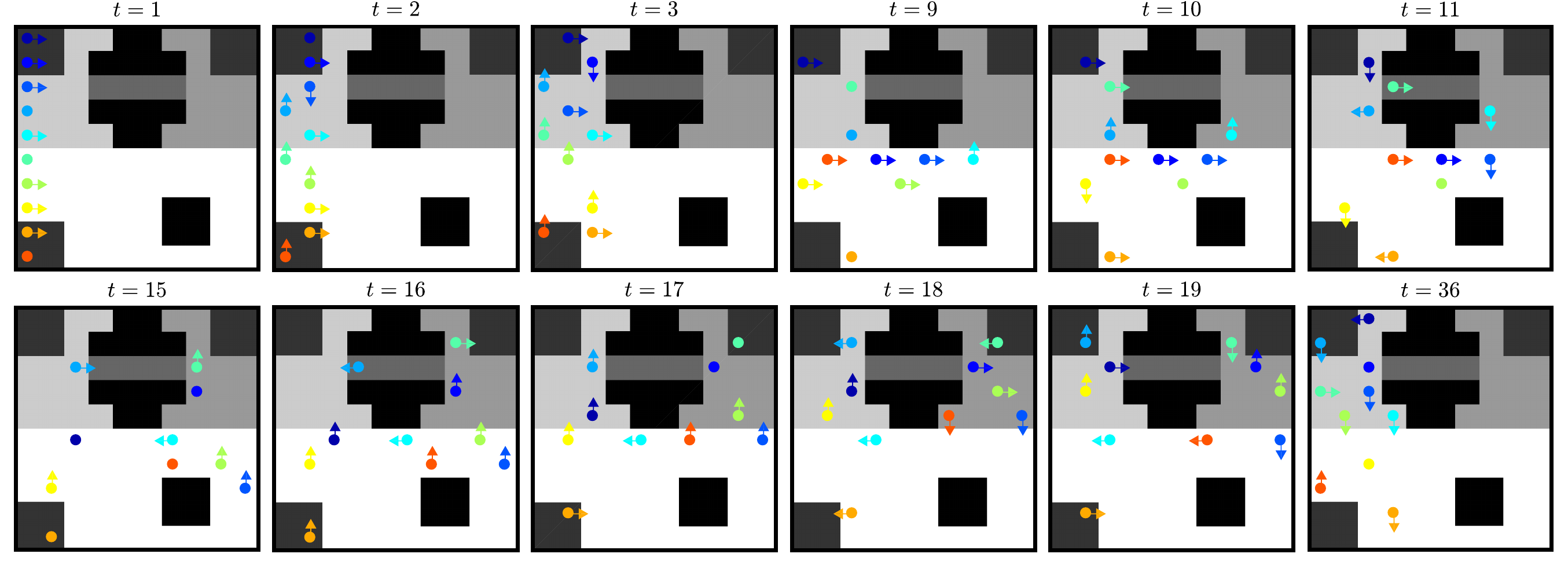} 
  \caption{Important frames from the synthesized non-robust trajectories, where arrows indicate direction of movement. The loop starts at frame $t=4$, thus the state at $t=4$ is identical to the state at $t=36$. Time $t=16$ and $t=18$ are the only time steps where region $A$ and $C$ are emptied and populated with more than 5 robots, respectively. The bridge is empty until two robots inspect it from different sides at $t=11$. Every robot visits the charging station and avoids collisions. }\label{fig:solution} 
\end{figure*}

\begin{figure*}[t] 
  \centering 
  \includegraphics[width=.92\linewidth]{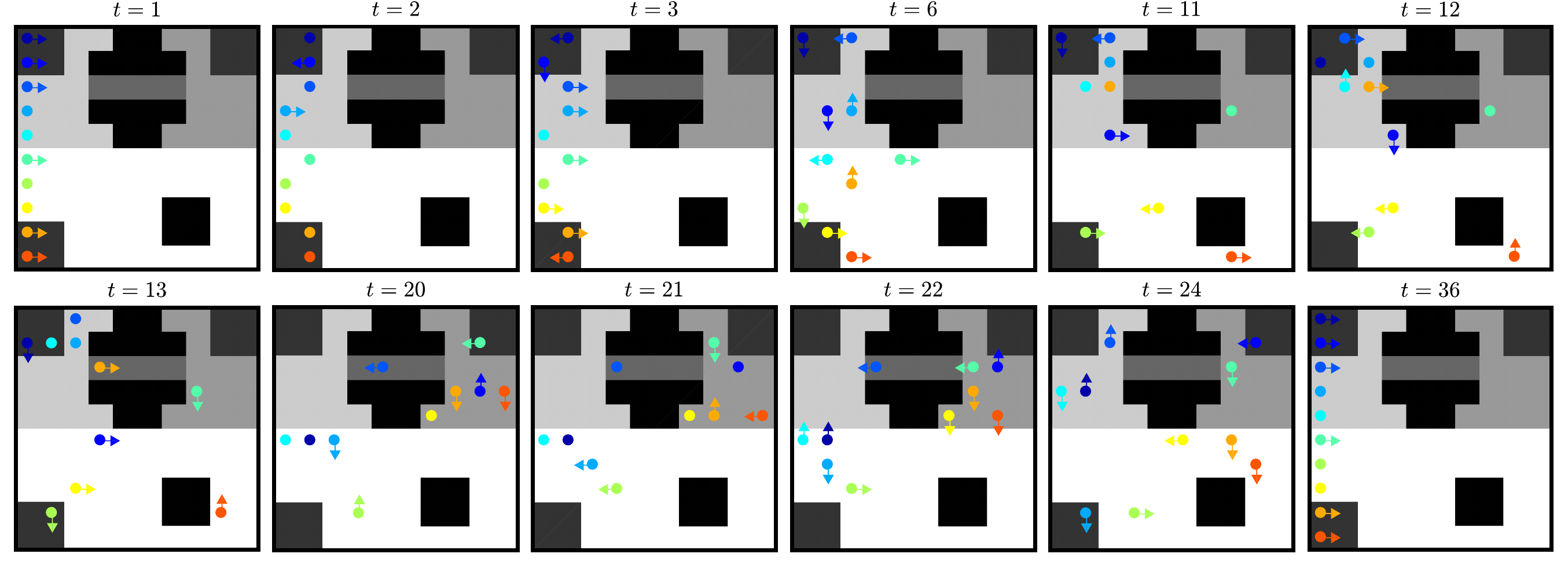} 
  \caption{Important frames from the synthesized robust trajectories, where arrows indicate direction of movement. The loop starts at frame $t=1$, which is identical to frame $t=36$. The number of robots in region $A$ ($C$) is $5$ ($0$) between $t=1$ and $t=3$ and $0$ ($5$) between $t=20$ and $t=22$, which implies that $\mu_4$ to $\mu_7$ are robustly satisfied at anchor time $t=1$. No robots use the narrow passage until it has been examined by both sides between $t=11$ and $t=13$, and the number of robots on the bridge never exceeds $2$; hence $\mu_2$ and $\mu_8$ are robustly satisfied.
    Every robot visits the charging station and avoids collisions.}\label{fig:robust_solution}
\end{figure*}

\subsection{Numerical examples}

To examine the scalability of the proposed approach, we use the emergency response example explained in the previous section as a base example with the following parameters: the number of robots $N=10$, solution horizon $h= 35$ and robustness parameter $\tau=0$. We then vary one of these parameters at a time and report the average solution times over $5$ runs in Table \ref{tab:results}. 

We report results for three different implementations in Table \ref{tab:results}. The first implementation uses the encodings proposed in this paper. The second implementation is a special encoding that can only be used for $4$-connected grid environments. That is, robots move in a two dimensional gridded environment only horizontally or vertically. In this implementation, the number of Boolean variables needed to denote the state of the robot on a $x\times y$ gridded environment is $x+y$ as opposed to $xy$ for a general implementation. A smaller number of decision variables decreases the solution times significantly. We also implement the continuous-state extension proposed in Section \ref{ssec:continuos}. As can be seen in Table \ref{tab:results}, solution times can be reduced significantly if the encodings that are most appropriate for the problem at hand are used.

Additionally, we examine the solution times for different encodings when specifications are given in cLTL and the robots have identical dynamics. Assume that the transition system $T = (S, \ra, AP, L)$, where $\ra$ is generated from an Erd{\"o}s-R{\'e}nyi graph with edge probability 0.25, represents the dynamics of $N$ robots. The set $S$ of states is partitioned into two sets of same size and labeled with $s_1\in AP$ and $s_2\in AP$, Each robot is assigned an initial state that is randomly selected from those labeled with $s_1$. Three goal regions are created such that each has $\frac{|S|}{10}$ randomly selected states and are labeled with $g_i \in AP$ for $i=1,2,3$.
The specification is given by the cLTL formula $\mu$:
\begin{equation}
\mu = \left( \eventually\always [s_2, N/2] \right) \land \left( \bigwedge_{i=1}^3 \always\eventually[g_i,N/3]\right).
\end{equation}

The specification $\mu$ requires at least half of the robots to reach states marked by $s_2$ and stay there indefinitely. Also, each goal region must be populated by at least $N/3$ robots, infinitely often over time. The results in Table \ref{tab:results2} are obtained by varying either the number of robots $N=10$ or the time horizon $h=20$ while keeping all the other parameters intact. Solution times in the first and second column are obtained by alternative cLTL encodings proposed in Section \ref{chp:solcLTL} and regular cLTL+ encodings, respectively. Regular cLTL+ encodings could not find solutions for $N=500$ within the timeout threshold of $60$ minutes. On the other hand, cLTL encodings scale much better with the number of robots and easily handle hundreds of robots in a matter of seconds. In fact, solution times are almost unaffected by the number of robots.

\begin{table}
  \centering 
  \caption{Numerical results}  \label{tab:results}
  \renewcommand{\arraystretch}{1.2}
  \begin{tabular}{c c|c|c|c}
    & & \textbf{cLTL+} & \textbf{cLTL+} & \textbf{cLTL+}\\ 
    && \textbf{(regular)} & \textbf{(grid)} & \textbf{(continuous)}  \\ \hline
    \multirow{5}{*}{\textbf{N}}& $4$ & $10.54$  & $9.74$ & $72.04$\\
    &$6$ & $33.86$  & $17.87$ & $86.47$  \\ 
    &$8$ & $1974.7$  & $155.36$ & $2286$  \\ 
    &$10$ & $992.02$  & $265.38$ & $132.18$  \\\hline
    \multirow{7}{*}{\textbf{h}}& $30$ & $389.97$  & $853.75$ & $108.87$\\
    &$35$ & $992.02$  & $265.38$ & $132.18$  \\ 
    &$40$ & $2788.1$  & $556.51$ & $181.94$  \\ 
    &$45$ & $1842.2$  & $201.45$ & $274.3$  \\
    &$50$ & $2505$  & $238.86$ & $257.77$  \\ 
    &$55$ & $2436.3$  & $334.97$ & $364.48$  \\ 
    &$60$ & $3828.1$  & $266.19$ & $440.68$  \\ \hline
    \multirow{3}{*}{\textbf{$\tau$}}&$0$ & $992.02$  & $265.38$ & $132.18$ \\ 
    &$1$ & $9367.1$  & $507.96$ & $829.41$ \\             &$2$ & $34222.1$  & $323.74$ & $18234$\\ \hline
  \end{tabular}
\end{table}

\begin{table}
  \centering 
  \caption{Numerical results}  \label{tab:results2}
  \renewcommand{\arraystretch}{1.2}
  \begin{tabular}{c c|c|c}
    & & \textbf{cLTL} & \textbf{cLTL+}  \\  \hline
    \multirow{5}{*}{\textbf{N}}& $10$ & $10.86$  & $2.64$\\
    &$20$ & $10.12$  & $5.87$  \\ 
    &$50$ & $8.99$  & $56.13$   \\ 
    &$500$ & $12.72$  & $TO$ \\\hline
    \multirow{3}{*}{\textbf{h}}& $20$ & $10.86$  & $2.64$\\
    &$40$ & $26.84$  & $5.32$  \\ 
    &$60$ & $60.93$  & $7.87$ \\ \hline
  \end{tabular}
\end{table}

\section{Conclusions}

In this paper we presented counting temporal logics (cLTL and cLTL+) that are convenient for specifying desired behaviors for multirobot systems. We also proposed an optimization-based trajectory generation method to synthesize collective behaviors that satisfy specifications given in these formalisms. Furthermore, we showed how to generate trajectories that are robust to bounded asynchrony. We then discussed how to handle continuous-state systems and extended the cLTL+ syntax so that tasks can be assigned to a subset of robots. As numerical results suggest, solution times depend greatly on the specific method for encoding specifications. One possible direction for future research is to discover relevant applications and develop encodings tailored specifically to them. Finally, while the proposed techniques are shown to scale well with the number of robots, scalability with respect to the size of the transition system of the individual robots and with respect to the robustness parameter $\tau$ remains a challenge, which we are working on addressing via hierarchical approaches \cite{sahin2018multi}.

\ifCLASSOPTIONcaptionsoff
  \newpage
\fi

\bibliographystyle{abbrv}
\bibliography{ref.bib}

\appendix

\subsection{Proof of Theorem \ref{prop:robustness_is_sound}}

First of all, note that $\zr{n}{\phi}t = 1$ if and only if  $\z{n}{\phi}{t+k} = 1$ for all $k \in [0,\tau]$ due to \eqref{eq:ap_robust{t}}. That is, $\zr{n}{\phi} t = 1$ implies that robot $\mathcal{R}_n$ satisfies the inner formula $\phi$ for $\tau+1$ consecutive steps, starting from time $t$. By the restriction of formulas to PNF, it is enough to prove the soundness for the operators in \eqref{eq:logic_restr} and we do so recursively, starting with temporal counting propositions. 

\emph{tcp:}
Let $\mu = [\phi, m] \in \Phi \times \mathbb{N}$ and a collection $ \Pi = \{\traj{1}, \dots, \traj{N}\}$ of trajectories be given. We first show that $\y{\mu}{t} = 1$ implies that $\Pi$ $\tau$-robustly satisfies $\mu$ at anchor time $t$. Assume $m> 1$ and $\y{\mu}{t} = 1$. Then $\sum_{n=1}^N \zr{n}{\phi}{t} \geq m$ due to \eqref{eq:cp_robust}. Without loss of generality, assume that robots are enumerated such that the first $m$ robots robustly satisfy $\phi$ at time step $t$, i.e., $\zr{n}{\phi}{t} = 1$ for all $n \in [m]$. Then $\z{n}{\phi}{t+k} = 1$ for all $n \in [m]$ and for all $k \in [0, \tau]$ due to equation \eqref{eq:ap_robust{t}}. Now let $K \in \mathcal{K}_N(\tau)$ be an arbitrary $\tau$-bounded execution and $T$ be an arbitrary time step with anchor time $t$, i.e., $b_K(T)=t$. By definition of $\tau$-bounded executions, local times are restricted to $t\leq \ktn n T \leq t+\tau$. Then, $\sum_{n=1}^N \z{n}{\phi}{\ktn n T} \geq \sum_{n=1}^m \z{n}{\phi}{\ktn n T} = m$. Hence, $(\Pi, K), T \models_\tau \mu$. Note that this is true for all  for all $K \in \mathcal{K}_N(\tau)$ and for all $T\in b_K^{-1}(t)$. Thus, $\Pi,t \models_\tau \mu$ by definition of robust satisfaction.

Now assume $m=1$ and $\y{\mu}{t} = 1$. Due to \eqref{eq:cp_robust1}, either $\sum_{n=1}^N \zr{n}{\phi}{t} \geq 1$ or $\sum_{n=1}^N \z{n}{\phi}{t} = N$. If the former is true, earlier arguments apply. Then, assume the latter is true, that is, $\z{n}{\phi}{t} =1$ for all $n$. Let $K=[k_1\dots k_N]^T \in \mathcal{K}_N(\tau)$ be arbitrary. At anchor time $t$, there exists at least one robot such that $\ktn n T = t$. Without loss of generality assume $\ktn 1 T = t$. Then $\sum_{n=1}^N \z{n}{\phi}{\ktn n T} \geq \z{n}{\phi}{\ktn 1 T} = \z{n}{\phi}{t} = 1$, hence $\Pi,t \models_\tau \mu$. These arguments hold for any $t$, including $t=0$, hence the modified encodings in \eqref{eq:cp_robust}-\eqref{eq:cp_robust1} are sound for temporal counting propositions.

\emph{conjunction:} Showing soundness for conjunction is straightforward. Assume $\mu = \bigwedge \mu_i$ and for a collection $\Pi = \{\traj{1}, \dots, \traj{N}\}$, $\y{\mu}{t} = 1$ for some $t$. Then $\y{\mu_i} t = 1$ for all $i$, implying that $\Pi,t \models_\tau \mu_i$. In other words, for all $K\in \mathcal{K}_N(\tau)$ and for all $T \in b^{-1}_K(t);\; (\Pi,K),T \models \mu_i$ for all $i$. Hence $\Pi,t \models_\tau \mu$.

\emph{disjunction:} Let $\mu = \bigvee_i \mu_i$ and $\y{\mu}{t} = 1$ for some $t$ and some collection $\Pi = \{\traj{1}, \dots, \traj{N}\}$. Since disjunction is associative and commutative, we can rewrite $\mu = \mu_{tcp} \lor \mu_o$ where $\mu_{tcp} = \bigvee_i [\phi_i,m_i]$ is conjunction of $tcp$ and $\mu_o$ is the disjunction of the rest of the clauses that are not $tcp$. We first show that encoding of disjunction of temporal counting propositions is sound. If $\y {\mu}t = 1$, then either $\y {\mu_i}t = 1$ for some $i$, or $\sum_{n=1}^N \zr{n}{(\bigvee_{i} \phi_i)} t > \sum_{i} (m_i-1)$. If it is the former, $\y {\mu_i}t = 1$ for some $\mu_i = [\phi_i, m_i]$, then it follows from the soundness of $tcp$ encodings that $\Pi,t \models_\tau \mu_i$. Thus $\Pi,t \models_\tau \mu$. Now assume $\y{\mu_i}{t} =0$ for all $i$ and $\sum_{n=1}^N \zr{n}{(\bigvee_{i} \phi_i)} t > \sum_{i} (m_i-1)$. Note that $\zr{n}{(\bigvee_{i} \phi_i)}t = 1$ implies that for each $k \in [0,\tau]$, there exists at least one $\phi_i$ such that $\traj{n},t+k\models \phi_i$. Now for arbitrary set of local indices $\{\ktn n T\}$ such that $t\leq \ktn n T \leq t+\tau$, let $\tilde m_i$ be the number of robots who satisfy $\phi_i$, i.e., $\tilde m_i \doteq |\{n \mid  \z{n}{\phi_i}{\ktn n T }=1\}|$. Then $\sum_i \tilde m_i \geq \sum_{n=1}^N \zr{n}{(\bigvee_{i} \phi_i)}t > \sum_{i} (m_i-1)$. Note that if $\tilde m_i< m_i$ for all $i$, the last inequality cannot be true. Hence, there exists at least one $\tilde m_i \geq m_i$. As a result, $\Pi,t\models\mu_i$ for at least one $\mu_i$ and $\Pi,t\models\mu$. 

Showing soundness of \eqref{eq:dis_robust2} is straightforward and omitted here. All of these combined together proves the correctness of \eqref{eq:dis_robust} and \eqref{eq:dis_robust2}.

\emph{until:} Until encodings are quite close to standard encodings but the modification is needed due to change in disjunction encodings. Let $\eta = \mu_1 \until \mu_2$ and $\Pi = \{\traj{1}, \dots, \traj{N}\}$ be a collection. If $\y{\mu_2}t = 1$ for $\Pi$ and some $t$, then $\Pi,t\models\mu_2$ and $\Pi,t\models\eta$. Now assume $\y{\mu_2}t \not= 1$. The first line in equation \eqref{eq:until_robust} requires $\y{\mu_1 \lor \mu_2}t= 1$ and $\y\eta{t+1}=1$, for $\y\eta t=1$ to hold. Then $\Pi,t \models_\tau \mu_1\lor \mu_2$ and $\Pi, t+1\models_\tau \mu_1\until \mu_2$. This implies that $\Pi, t\models_\tau \mu_1\until \mu_2$. Similar to standard encodings, auxiliary variables are used to avoid trivial satisfaction and make sure $\mu_2$ is satisfied at some point.

Proving that the ``release'' operator encodings are also sound is similar to ``until'' case and omitted here.

We showed that outer logic encodings are sound. The soundness of the whole encoding procedure follows as before from soundness of ILP encodings of LTL, which is used for inner logic formulas. \qed 
\vspace{-3mm}

\subsection{Proof of Theorem \ref{prop:robustness_is_complete}}

We first give an {outline} of the proof and then provide details. The proof starts by showing that the modified encodings are complete for the simplest specification, $\mu = [\phi,m]$. We then show that conjunction and next operators preserve completeness. Next, we show that disjunction and until operators are complete for mutually exclusive atomic propositions. That is enough to prove Theorem \ref{prop:robustness_is_complete} due to the special form of specifications and the second assumption that atomic propositions are mutually exclusive. We now give details of these steps.

\emph{tcp:} Let $\mu = [\phi,m]$ be a temporal counting proposition and $\Pi = \{\traj{1}, \dots, \traj{N}\}$ be a collection such that $\Pi,t \models_\tau \mu$ for some $t$. We are going to show that if \eqref{eq:cp_robust} (or \eqref{eq:cp_robust1} for $m=1$) does not hold for some $t$, then $\Pi,t \not\models_\tau \mu$. First assume $m>1$ and $\sum_{n=1}^N \zr{n}{\phi}t < m$. Assume without loss of generality that robots are enumerated such that $\zr{n}{\phi}t = 0$ at least for the first $N-m+1$ robots. Then, for all $n \in [N-m+1]$, there exist at least one $\z{n}{\phi}{t+k} = 0$ for some $k \in \{0,1,\dots,\tau\}$. Assume each $\hat{t}_n$ denotes the first instance where $\z{n}{\phi}{\hat{t}_n} = 0$ for $\hat{t}_n \in [t,t+\tau]$ and for all $n \in [N-m+1]$. Then, there exists a $\tau$-bounded execution $K=[k_1\dots k_N]^T \in \mathcal{K}_N(\tau)$ such that $\ktn n T = {\hat{t}_n} $ for all $n \in [N-m+1]$ and $\ktn{N}T = t$ for some $T$. Note that such a $\Pi$ violates \eqref{def:robust_satisfaction} and creates a contradiction. Thus $\sum_{n=1}^N \zr{n}{\phi}t \geq m$ must hold.

In the special case when $m=1$, further assume that $\sum_{n=1}^N \z{n}{\phi}t < N$. This implies that, for each $n\in [N]$, $\z{n}{\phi}{\ktn n T} =0$ for some $T$ where $t \leq \ktn n T \leq t+\tau$ and there exists at least one robot $\tilde n$ such that $\z{\tilde n}{\phi}{t}=0$. Then choose $\ktn{\tilde n}T = t$ and for all other robots choose $\ktn{n}T$ such that $\z{n}{\phi}{\ktn n T} =0$. These set of indices have the anchor time $t$ and satisfy the $\tau$-boundedness criteria. Hence, there exists a $\tau$-bounded asynchronous execution such that $\mu$ is not satisfied. But this is a contradiction. Thus either  $\sum_{n=1}^N \zr{n}{\phi}t \geq 1$ or $\sum_{n=1}^N \z{n}{\phi}t = N$ must hold.

\emph{disjunction:} For the sake of ease, we show that \eqref{eq:dis_robust} is complete for disjunction of two temporal counting propositions. Let $\mu_i = [\phi_i, m_i]$ for $i=1,2$ and $\mu = \mu_1 \lor \mu_2$. Assume that \eqref{eq:dis_robust} fails to hold for some $t$, but that there exists a collection $\Pi = \{\traj{1}, \dots, \traj{N}\}$ such that $\Pi,t \models_\tau \mu$. This implies that, for all local time permutations with anchor time $t$, i.e., $\ktn n T \in [t, t+\tau]$ and $\min \ktn n T = t$, we have $\sum_{n} \z{n}{\phi_i}{\ktn n T} \geq m_i$ for either $i = 1$ or $i = 2$. Since \eqref{eq:dis_robust} fails to hold, we have $\sum_{n=1}^N \zr{n}{(\phi_1\lor \phi_2)}t  < m_1 + m_2 - 1$ which implies that $\sum_{n=1}^N \zr{n}{ \phi_i}t  < m_i$ for $i=1,2$. Now without loss of generality, enumerate robots such that $\zr{n}{(\phi_{1}\lor \phi_2)}{t}=1$ only for the first $n_{12}$ robots. This implies that, for the rest of the robots, one can choose a local time where both $\phi_1$ and $\phi_2$ fails to hold. Furthermore, assume that $\zr{n}{\phi_{1}}{t}$ holds for the first $n_{1}$ robots and that $\zr{n}{\phi_{2}}{t}$ holds for the following $n_{2}$ robots. Since $AP$ are mutually exclusive, no robot can satisfy $\phi_1$ and $\phi_2$ at the same time. Then, starting from the $(n_1+n_2+1)^{th}$ robot, choose as local times the first $m_1-n_1-1$ such that $\z{n}{\phi_{1}}{\ktn n T}=1$. For the rest of the robots, until $n_{12}^{th}$, choose local times such that $\z{n}{\phi_{2}}{\ktn n T}=1$. Note that such selection always exists. Then $\sum_{n} \z{n}{\phi_{1}}{\ktn n T} = m_1-1,$ and $\sum_{n} \z{n}{\phi_{2}}{\ktn n T} = r_{12} - \sum_{n} \z{n}{\phi_{1}}{\ktn n T} = r_{12} - (m_1 -1) < m_1 + m_2 - 1 - (m_1 -1)\nonumber < m_2.$

Note that we can always choose $\ktn 1 T = t$. This is contradictory to the assumption that $\mu$ is $\tau$-robustly satisfied. Thus, we conclude that \eqref{eq:dis_robust} is necessary for $\mu$ to be satisfied.

Completeness for conjunction and next operators follows from the completeness of standard ILP encodings for bounded model checking. Completeness of until operator follows from completeness of conjunction and disjunction operators.
\qed

\end{document}